\PassOptionsToPackage{table}{xcolor}
\pdfoutput=1
\documentclass[twoside]{article}

\usepackage[accepted]{aistats2026_arxiv}

\usepackage[round]{natbib}

\usepackage{booktabs}
\usepackage{enumitem}
\usepackage{mathtools}
\usepackage{amsmath,amsfonts,amssymb,amsthm, bm}
\usepackage{thmtools}
\usepackage{url}
\usepackage{graphicx}
\usepackage{float}
\usepackage{wrapfig}
\usepackage{subcaption}
\usepackage{xfrac}
\usepackage{xcolor}
\usepackage{pifont}
\usepackage{tikz}
\usetikzlibrary{positioning,calc,arrows.meta}%
\definecolor{mydarkblue}{rgb}{0,0.08,0.45}
\usepackage[colorlinks,citecolor=mydarkblue,urlcolor=mydarkblue,linkcolor=mydarkblue,filecolor=mydarkblue]{hyperref}
\usepackage{cleveref}
\usepackage[makeroom]{cancel}

\usepackage{siunitx}
\newcommand{\mathbold}[1]{\bm{#1}}
\newcommand{\mbf}[1]{\mathbf{#1}}

\newcommand{\T}{\top}%
\newcommand{\R}{\mathbb{R}}%

\DeclareMathOperator*{\argmin}{arg\,min}
\DeclareMathOperator*{\argmax}{arg\,max}

\newcommand{\norm}[1]{\left\lVert#1\right\rVert}

\newcommand{\vtheta}[0]{\mathbold{\theta}}

\newcommand{\MPhi}[0]{\mathbold{\Phi}}

\renewcommand{\mid}[0]{\,|\,}

\newcommand{\vb}{\mbf{b}}

\newcommand{\vg}{\mbf{g}}

\newcommand{\vk}{\mbf{k}}

\newcommand{\vr}{\mbf{r}}

\newcommand{\vx}{\mbf{x}}

\newcommand{\vz}{\mbf{z}}
\newcommand{\MA}{\mbf{A}}
\newcommand{\MB}{\mbf{B}}
\newcommand{\MC}{\mbf{C}}

\newcommand{\MH}{\mbf{H}}
\newcommand{\MI}{\mbf{I}}

\newcommand{\MK}{\mbf{K}}

\newcommand{\MM}{\mbf{M}}

\newcommand{\MS}{\mbf{S}}

\newcommand{\bigo}{\mathcal{O}}

\newcommand{\rnd}[1]{\left(#1\right)}
\newcommand{\sqr}[1]{\left[#1\right]}

{}
\newcommand{\data}{\mathcal{D}}

\newcommand{\vphi}{\boldsymbol{\phi}}

\newcommand{\entropy}{\mathcal{H}}
\newcommand{\activeset}{\mathcal{A}}

\usepackage{xspace}
\newcommand{\eg}{\textit{e.g.}\@\xspace}
\newcommand{\ie}{\textit{i.e.}\@\xspace}

\crefname{section}{Sec.}{Sec.}
\crefname{lemma}{Lemma}{Lemmas}
\crefname{thm}{Thm.}{Theorem}
\crefname{appendix}{App.}{Appendices}
\crefname{algorithm}{Alg.}{Algorithms}
\crefname{equation}{Eq.}{Eqs.}
\crefname{figure}{Fig.}{Figs.}
\crefname{table}{Table}{Table}

\newcommand\blfootnote[1]{%
  \begingroup
  \renewcommand\thefootnote{}\footnote{#1}%
  \addtocounter{footnote}{-1}%
  \endgroup
}

\begin{document}

\runningauthor{Dharmesh Tailor, Nicolò Felicioni, Kamil Ciosek}

\twocolumn[

\aistatstitle{A Bayesian Information-Theoretic Approach to Data Attribution}

\aistatsauthor{ Dharmesh Tailor$^{*}$ \And Nicolò Felicioni \And  Kamil Ciosek }

\aistatsaddress{ University of Amsterdam \And  Spotify \And Spotify } ]

\begin{abstract}
Training Data Attribution (TDA) seeks to trace model predictions back to influential training examples, enhancing interpretability and safety. We formulate TDA as a Bayesian information-theoretic problem: subsets are scored by the \emph{information loss} they induce—the entropy increase at a query when removed. This criterion credits examples for resolving predictive uncertainty rather than label noise. To scale to modern networks, we approximate information loss using a Gaussian Process surrogate built from tangent features. We show this aligns with classical influence scores for single-example attribution while promoting diversity for subsets. For even larger-scale retrieval, we relax to an information-gain objective and add a variance correction for scalable attribution in vector databases. Experiments show competitive performance on counterfactual sensitivity, ground-truth retrieval and coreset selection, showing that our method scales to modern architectures while bridging principled measures with practice.
\end{abstract}

\begin{figure*}[t]
    \centering
    \begin{minipage}[t]{0.40\textwidth}
        \centering
        \begingroup
\renewcommand{\arraystretch}{1.1}

\def\cmark{\ding{51}}%
\def\xmark{\ding{55}}%
\def\head{\small\bfseries\sffamily}%
\def\yes{\color{mydarkblue}}%
\def\no{\color{black!65}}%
\def\warn{\color{orange!85!black}}%

\begin{tikzpicture}[
    font=\scriptsize\sffamily,
    >=Latex,
    box/.style={
        draw,
        rounded corners=3pt,
        thick,
        align=left,
        inner xsep=4pt,
        inner ysep=3pt,
        text width=5.35cm,
        fill=black!2
    },
    pill/.style={
        rounded corners=2pt,
        inner sep=3pt,
        font=\scriptsize\sffamily,
        align=left
    },
    approxArrow/.style={->, thick, densely dashed, black!65},
    regimeLabel/.style={pill, text=black!75}
]

\node[box] (IL) {
    {\head Information Loss}\\[1pt]
    \begin{tabular}{@{}l l@{}}
    \textbf{Submodular?} &
        {\no \textcolor{black!65}{\xmark\ \textbf{No}}} \\
    \textbf{Vector query?} &
        {\no \textcolor{black!65}{\xmark\ \textbf{No}} (depends on $D\!\setminus\!S$)} \\
    \end{tabular}
};

\node[box, below=6mm of IL] (IG) {
    {\head Information Gain}\\[1pt]
    \begin{tabular}{@{}l l@{}}
    \textbf{Submodular?} &
        {\yes \textcolor{mydarkblue}{\cmark\ \textbf{Yes}} (greedy $(1-1/e)$)} \\
    \textbf{Vector query?} &
        {\no \textcolor{black!65}{\xmark\ \textbf{No}} (candidate-dep.\ denom.)} \\
    \end{tabular}
};

\node[box, below=6mm of IG] (ApproxIG) {
    {\head Approx.\ Information Gain}\\[-1pt]
    {\scriptsize (linear-response variance correction)}\\[1pt]
    \begin{tabular}{@{}l l@{}}
    \textbf{Submodular?} &
        {\warn \textbf{No guarantee} (approx.)} \\
    \textbf{Vector query?} &
        {\yes \textcolor{mydarkblue}{\cmark\ \textbf{Yes}} (squared inner prod.)} \\
    \end{tabular}
};

\draw[approxArrow] (IL.south) -- (IG.north);
\path (IL.south) -- (IG.north) node[regimeLabel, midway, right=1mm] {
    high-noise regime\\[-1pt]
    (IG $\approx$ IL)
};

\draw[approxArrow] (IG.south) -- (ApproxIG.north);
\path (IG.south) -- (ApproxIG.north) node[regimeLabel, midway, right=1mm] {
    linear-response\\[-1pt]
    (Approx.\ IG $\approx$ IG)
};

\end{tikzpicture}
\endgroup
    \end{minipage}%
    \hspace{0.01\textwidth}%
    \begin{minipage}[t]{0.58\textwidth}
        \centering
        \input{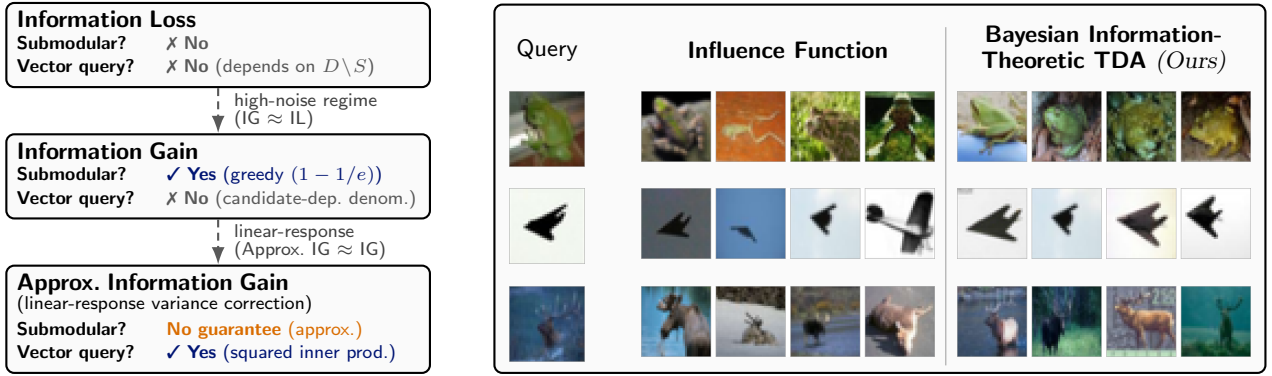}
    \end{minipage}
    \caption{Bayesian information-theoretic training data attribution (TDA). For a query input $\vx_*$, we quantify the contribution of a training subset $S$ via its \emph{information loss}: the increase in (posterior) predictive entropy at $\vx_*$ when $S$ is withheld from training, attributing credit to examples that resolve \emph{epistemic} uncertainty rather than label noise. Left: we relate information loss to a submodular \emph{information gain} relaxation whose leading-order term matches information loss in a high-noise regime, followed by a linear-response variance correction that implements greedy selection with a squared inner-product score, enabling efficient retrieval in a vector database. Right: CIFAR-10 (ResNet-9) examples showing the top-ranked training images under an influence-function estimator versus our information loss criterion.}
    \label{fig:teaser}
\end{figure*}

\section{INTRODUCTION}
The problem of training data attribution (TDA) is to trace a model's prediction back to its training data, thereby addressing a fundamental need for interpretability and transparency in modern machine learning systems. In contrast to approaches that seek to understand a model's internal workings in a \emph{mechanistic} fashion, TDA is a form of explanation that is grounded in data. Attributed training examples can be used to provide explanations for recommendation systems, enable risk mitigation strategies that address safety, privacy, and fairness concerns, and facilitate appropriate compensation frameworks for creative work while detecting potential plagiarism or memorization.

The de-facto technique for TDA is influence function---a classical tool from robust statistics \citep{hampel1974influence} and regression diagnostics \citep{cook1980characterizations} revived for modern machine learning by \citet{koh2017understanding}. 
By linearizing the stationarity condition at the trained parameters, influence function provides a local, first-order estimate of the counterfactual change in a model's loss or prediction when a training point is infinitesimally upweighted or removed, thereby sidestepping costly retraining.%
This convenience comes with caveats: the derivation depends on fully-converged parameters which is not a realistic assumption in practice; and requires the computation and inversion of large curvature matrices which is impractical for modern architectures.
A substantial body of work has since tackled some of these concerns by proposing alternate influence function estimators that relax such restrictions or exploiting numerical approximations for computational efficiency. 
This includes approaches that backtrack through training trajectories \citep{hara2019data}; gradient-accumulation methods like TraceIn \citep{pruthi2020estimating}; expressive curvature surrogates such as EK-FAC \citep{grosse2023studying}; leveraging random projections or improved linear system solvers \citep{schioppa2022scaling}.
\blfootnote{ $^{*}$ Part of this work was carried out when Dharmesh Tailor was on internship at Spotify.}

In this work, we replace the loss-based counterfactual notion of TDA with an explicitly information-theoretic one: for a test query we score a subset by the \emph{information} lost about the latent prediction when that subset is withheld---\ie, the increase in entropy at the query (measured in nats). This choice targets epistemic uncertainty attributable to missing training data, and places our formulation in a broader tradition of Bayesian and information-theoretic approaches to learning. Information measures have long been used to formalize principled learning objectives and guarantees---\eg, information gain as an acquisition rule in Bayesian experimental design and active learning \citep{mackay1992information,houlsby2011bayesian}, information-theoretic analyses of generalization via mutual information \citep{xu2017information,steinke2020reasoning}, and the information bottleneck view of representation learning and training dynamics \citep{shwartz2017opening}. Closer in spirit, Bayesian viewpoints have studied TDA under training stochasticity, treating attribution scores as random variables rather than fixed point estimates \citep{nguyen2023bayesian}, and data point selection through posterior inference over instance-wise weights \citep{xu2024bayesian}. These perspectives are complementary to our query-specific information-theoretic formulation of subset attribution.

\paragraph{Contributions.}
\textit{(i)} We pose TDA as \emph{information loss} at a query, the entropy increase induced by withholding a subset. This can be implemented by instantiating a Gaussian Process surrogate from tangent features of the network (the empirical NTK), leading to a closed-form expression. 
\textit{(ii)} We show that this criterion admits a principled \emph{relaxation to information gain}: for fixed subset size and large observation noise, information loss and information gain share the same leading-order expansion. The score for a subset then depends only on the selected points themselves, not on the remaining training data.
\textit{(iii)} To make our algorithm amenable to fast lookup, we derive a \emph{linear-response variance correction} that turns each step into a squared inner-product search between a residual query vector and (precomputed) sketched Jacobians. All inversions are performed in a low-dimensional sketch space, and we demonstrate how selection can be implemented via approximate nearest-neighbour search over a vector database (querying with both the residual and its negation, then merging top candidates), giving a computational footprint comparable to modern influence-based pipelines. 
\textit{(iv)} Empirically, we evaluate three complementary settings: \emph{leave-subset-out brittleness}—counterfactual sensitivity of predictions when top-ranked subsets are removed; \emph{backdoor attribution}, where our method retrieves the ground-truth poisoned subset in controlled attacks; and \emph{coreset selection}, where we test whether attributed subsets preserve downstream accuracy when used for retraining.

\section{BACKGROUND}\label{sec:background}
\paragraph{Setup.}
We consider supervised learning with training data $\data \!=\! \{(\vx_i,y_i)\}_{i=1}^N$, covering tasks from image classification to next-token prediction in language modeling. Let $f_{\vtheta}$ denote a model with parameters $\vtheta \in \R^P$ that is trained on $\data$ (\eg by empirical risk minimization). 
Training data attribution (TDA) seeks to quantify the contribution of a subset $S \subset \data$ to the model's behaviour at a designated test query $\vx_*$ (optionally with label $y_*$). We write this contribution as an \emph{attribution score} $\mathcal{I}(S;\vx_*)$, intentionally agnostic about the precise notion of ``behaviour''—it may refer to loss, prediction, or another task-relevant functional of $f_{\vtheta}$.

\paragraph{A counterfactual ideal and its relaxations.}
Let $\hat{\vtheta}$ be the parameters obtained by training on $\data$ and $\hat{\vtheta}^{\setminus S}$ the parameters obtained when the subset $S$ is removed and the model is retrained under the same procedure. 
A natural ``ground-truth'' attribution for a loss-based notion of behaviour is the counterfactual change in test loss.
Denoting $\vz_* = (\vx_*,y_*)$, we have,
\begin{equation}
   \ell_*(\vtheta) := \ell(\vtheta;\vz_*), 
   \quad
   \mathcal{I}(S;\vx_*) \triangleq \ell_*(\hat{\vtheta}^{\setminus S}) - \ell_*(\hat{\vtheta}),
   \label{eq-inf}
\end{equation}
which measures how the example-wise loss at $\vx_*$ would differ if $S$ were not available for training. This choice requires access to $y_*$; when labels are unavailable one can replace the loss with a prediction-based quantity, \eg comparing model outputs $f_{\hat{\vtheta}^{\setminus S}}(\vx_*)$ and $f_{\hat{\vtheta}}(\vx_*)$ directly. Regardless of the choice of attribution score, computing \eqref{eq-inf} is expensive: identifying the most influential subset of size $M$,
\begin{equation}
   \argmax_{S \subset \data:\, |S|=M} \mathcal{I}(S;\vx_*),
\end{equation}
is combinatorial and intractable in general (except for the singleton case $M{=}1$). 
To avoid the cost, many practical TDA methods—most notably those based on first-order influence functions—adopt \emph{additive} group scores.
In this case, the subset attribution decomposes into pointwise terms and the size-$M$ maximizer is obtained by simple top-$M$ selection:
\begin{equation}
   \mathcal{I}_{\mathrm{IF}}(S;\vx_*) = \sum_{\vz_i \in S} \mathcal{I}_{\mathrm{IF}}(\{\vz_i\};\vx_*),
\end{equation}
where $\vz_i=(\vx_i,y_i)$.
This rank-and-pick rule is computationally attractive—pointwise scores can be precomputed once per test query and reused for any $M$—but it discards pairwise and higher-order interactions between training examples. The resulting redundancy (\eg over-counting near-duplicates) has motivated extensions that account for interactions, such as higher-order influence functions \citep{basu2020second}. 
We return to this limitation when introducing our information-theoretic criteria in the next section.

\section{INFORMATION-THEORETIC DATA ATTRIBUTION}
Training Data Attribution (TDA) is fundamentally about quantifying and allocating information: which subsets of data best contribute to resolving uncertainty. The most principled toolkit for measuring information is information theory, which provides a mathematically rigorous framework for quantifying uncertainty and information in nats. A common critique of such an approach is its perceived lack of tractability, since information-theoretic objectives often involve intractable conditioning on large datasets. However, as we will show, the approximations developed later in this section reduce the implementation to efficient primitives--queries in a vector database and linear algebra over small matrices--bringing the framework into practical reach. With this motivation in place, we now introduce the information-theoretic framework underlying our approach, beginning with an ideal criterion and then developing its practical relaxation.

We propose to score a subset $S \subset \data$ at a test query $\vx_*$ by the \emph{information loss} (InfoLoss),
\begin{equation}\label{eq:info_loss}
   \mathcal{I}_{\mathrm{IL}}(S;\vx_*)
   =
   -\entropy\!\big[f_* \mid \vx_*, \data\big]
   +
   \entropy\!\big[f_* \mid \vx_*, \data \setminus S\big],
\end{equation}
that is, the amount of information (in nats) about the latent $f_* \!:=\! f(\vx_*)$ that would be lost if $S$ were removed.
Under a Gaussian model that we consider in this work, this directly corresponds to an increase in marginal uncertainty about the test query.

\paragraph{A Bayesian surrogate.}
In common practice with pretrained networks (or networks trained using standard learning algorithms such as SGD), parameters are available as a point estimate $\hat{\vtheta}$. The mapping $\vx \mapsto f_{\hat{\vtheta}}(\vx)$ is therefore deterministic, rendering the entropy terms in \cref{eq:info_loss} degenerate. To obtain meaningful information-theoretic quantities, we use a \emph{Bayesian surrogate} that endows latent functions with uncertainty consistent with the local geometry of the trained model.

\paragraph{Neural networks as Gaussian processes.}
Our Bayesian surrogate is a Gaussian process induced by linearizing the network around $\hat{\vtheta}$:
\begin{equation}\label{eq:gp_ntk}
   k(\vx,\vx') = \nabla_{\vtheta} f_{\hat{\vtheta}}(\vx) \nabla_{\vtheta} f_{\hat{\vtheta}}(\vx')^\T
\end{equation}
where $k(\cdot,\cdot)$ -- often referred to as the (empirical) \emph{neural tangent kernel} (NTK) \citep{jacot2018neural} -- is a linear kernel with tangent features given by Jacobians of the neural network.
In contrast to the infinite-width NTK at \emph{initialization}---which converges to a deterministic kernel and remains constant during training under the NTK parameterization \citep{jacot2018neural,yang2019scaling}---evaluating the Jacobians at the \emph{trained} $\hat{\vtheta}$ ties the kernel to the representation the model actually uses. This is precisely what is justified by the linearized (Laplace) posterior around $\hat{\vtheta}$ \citep{khan2019approximate,immer2021improving} and its function-space counterpart \citep{pan2020continual}.

Since \cref{eq:info_loss} depends only on the entropy of the latent $f_*$ under the GP surrogate, it suffices to compute the marginal variance. For a Gaussian likelihood (squared-error regression) with noise variance $\sigma^2$, the standard Gaussian process regression (GPR) expressions apply \citep{williams2006gaussian}. Crucially, in this regression setting the marginal variance is \emph{independent of the targets}, so the entropy term—and therefore $\mathcal{I}_{\mathrm{IL}}$—depends only on inputs, the kernel and $\sigma^2$.

To keep the objective and its updates simple and to retain the Gaussian setting, we recast classification as regression (\eg with mean-centered one-hot targets) and use the GPR variance irrespective of the underlying likelihood. Under this surrogate, the entropy difference in \cref{eq:info_loss} reduces to a log-variance ratio.%

\paragraph{Relation to influence-based TDA.}
It turns out that in the singleton case ($|S|=1$), \cref{eq:info_loss} reduces to a form resembling a cross-influence term normalized by self-influence terms, \ie the cosine-style normalization used in \emph{RelatIF} \citep{barshan2020relatif}, especially when combined with a Gauss-Newton Hessian approximation.
In App.~\ref{app:info_loss_relation_influence}, we show RelatIF is label-independent (assuming a single-output case), which provides an explanation for its reported behaviour of selecting fewer ``globally'' influential examples that are often outliers or mislabelled examples.
However, an important difference is that RelatIF is \emph{signed}—it distinguishes positively vs.\ negatively influential examples—whereas our (singleton) score is \emph{unsigned}.

\paragraph{Relaxation to information gain.}
While the information loss criterion in \cref{eq:info_loss} is intuitively appealing, it has two critical issues. First, we have to condition on the near-complete dataset which is expensive. Second, information loss is not sub-modular and requires an intractable combinatorial optimization (we cannot expect a performance guarantee when optimising for it greedily). 
In principle, one could tackle the first problem in an analogous way to how recent linear system solvers handles Hessian-vector products (HVPs), that is replacing full-batch accumulations by sub-sampling as is done with the LiSSA scheme by \citet{agarwal2017second})\footnote{One could go even further and consider reducing the candidate space via heuristic prefiltering strategies \citep{grosse2023studying} or sparse kernel/GP techniques to obtain \emph{landmark} points (\eg leverage scores by \citet{alaoui2015fast}).}. However, the second problem is more challenging to address.

To sidestep these issues, we instead take a conceptually simpler approach and propose to use the \emph{information gain} (InfoGain) criterion. This scores a subset $S \subset \data$ at a test query $\vx_*$ by,
\begin{equation}\label{eq:info_gain}
   \mathcal{I}_{\mathrm{IG}}(S;\vx_*)
   =
   -\entropy\!\big[f_* \mid \vx_*, S\big]
   +
   \entropy\!\big[f_* \mid \vx_*\big],
\end{equation}
that is, the reduction (in nats) of the marginal uncertainty about $f_*:=f(\vx_*)$ if we trained only on $S$. Information gain is a standard acquisition criterion in Bayesian experimental design and active learning \citep{mackay1992information}; a key difference here is that we evaluate it directly at the test marginal rather than on parameters.

Crucially, for a fixed subset size and large observation noise, information loss and information gain share the same leading-order expansion. This makes InfoGain a principled high-noise relaxation of InfoLoss, formalized below.
\begin{restatable}[]{lemma}{highnoisequiv}
\label{lem:high_noise_equiv}
    Fix a query $\vx_*$ and subset size $M$, and let
    $\mathcal{S}_M := \{S \subseteq \data : |S| = M\}$.
    Then, as $\sigma^2 \to \infty$,
    \[
    \begin{aligned}
    \mathcal{I}_{\mathrm{IG}}(S;\vx_*)
    &=
    \frac{\norm{\vk_{S*}}^2}{2 k_{**}\sigma^2}
    +
    \bigo(\sigma^{-4}), \\
    \mathcal{I}_{\mathrm{IL}}(S;\vx_*)
    &=
    \frac{\norm{\vk_{S*}}^2}{2 k_{**}\sigma^2}
    +
    \bigo(\sigma^{-4}),
    \end{aligned}
    \]
    uniformly over $S \in \mathcal{S}_M$.
    In particular, if $\norm{\vk_{S*}}^2$ has a unique maximizer $S^\dagger$ over $\mathcal{S}_M$, then both criteria are maximized by $S^\dagger$ for sufficiently large $\sigma^2$.
\end{restatable}
For proof, see App.~\ref{app:info_criteria_obs_noise}. The lemma justifies using InfoGain as a high-noise surrogate for InfoLoss: we no longer need to condition on the whole dataset, and the new objective is sub-modular, allowing for the standard $1 - \frac{1}{e}$ guarantee for greedy algorithms \citep{nemhauser1978analysis}.

\paragraph{Greedy algorithm for information gain.}
Justified by its sub-modularity, we now optimize \cref{eq:info_gain} under the GP surrogate by greedy selection. 
Let $\activeset$ denote the set of acquired points (initially $\emptyset$). 
At each of the $M$ rounds we pick the example whose inclusion most reduces the marginal variance at $\vx_*$.
Using the formula for inverse of a partitioned matrix, we can express this with $\activeset$ as conditioning set throughout:
\begin{align}
   \vx^{(m)} &\;\leftarrow\; \argmin_{\vx \in \data \setminus \activeset} v_{*}^{\activeset \cup \{\vx\}} \label{eq:info_gain_greedy}
\end{align}
This is equivalent (see App.~\ref{app:greedy_info} for a full derivation) to
\begin{align}
    \vx^{(m)} &\;\leftarrow\; \argmax_{\vx \in \data \setminus \activeset} \frac{\rnd{{v_{\vx,*}^{\activeset}}}^2}{\sigma^2 + v_{\vx}^{\activeset}}, \label{eq:info_gain_greedy_rank1}
\end{align}
where the covariance terms are given by:
\begin{align}
   v_{\vx,*}^{\activeset} &= k_{\vx,*} - \vk_{\activeset, \vx}^\T \rnd{ \MK_{\activeset} + \sigma^2 \MI }^{-1} \vk_{\activeset, *}, \label{eq:cov_new_star_info_gain} \\
   v_{\vx}^{\activeset} &= k_{\vx,\vx} - \vk_{\activeset, \vx}^\T \rnd{ \MK_{\activeset} + \sigma^2 \MI }^{-1} \vk_{\activeset, \vx}. \label{eq:cov_new_new_info_gain}
\end{align}
Here $k_{\cdot,\cdot} \!=\! k(\cdot,\cdot)$ are kernel evaluations from \cref{eq:gp_ntk}, $\MK$ is the corresponding Gram matrix, and $\MI$ is the identity. Intuitively, the numerator is the squared (conditional) covariance between $f_{\vx}$ and $f_*$, and the denominator contains the marginal variance of the candidate point. After choosing $\vx^{(m)}$, we update $\activeset \leftarrow \activeset \cup \{\vx^{(m)}\}$ and iterate until $|\activeset|{=}M$.

\paragraph{Scalability via Jacobian sketches.}
To apply the greedy rule, we must repeatedly evaluate the covariance terms, which involves per-example Jacobians in $\R^{P}$ for the kernel evaluations and inversions whose dimension scales with $M$. 
While computing per-example Jacobians is tractable on modern accelerators, storing $O(N)$ of them is prohibitive in memory at contemporary scales (\eg pre-training corpora). We therefore adopt random-projection Jacobian sketches: draw a sketching matrix $\MA\in\R^{K\times P}$ with $K\ll P$ from a subgaussian family (\eg, Gaussian or Rademacher), and define
\begin{equation}
\begin{aligned}
   \vphi_{\vx} &= \sfrac{1}{\sqrt{K}}\MA\nabla_{\vtheta} f_{\hat{\vtheta}}(\vx)^\T \\
   &\text{so that} \quad k(\vx,\vx') \approx \vphi_{\vx}^{\!\top}\vphi_{\vx'} ,
\end{aligned}
\end{equation}
which (by Johnson–Lindenstrauss) preserves pairwise inner products with high probability. Substituting these features into the ``weight-space'' form (via Woodbury formula) yields the scalable approximations to \cref{eq:cov_new_star_info_gain,eq:cov_new_new_info_gain},
\begin{equation}\label{eq:cov_il_ws}
   v_{\vx,*}^{\activeset} \approx \vphi_{\vx}^{\!\top}\MS_{\activeset}^{-1}\vphi_{*},
   \quad
   v_{\vx}^{\activeset} \approx \vphi_{\vx}^{\!\top}\MS_{\activeset}^{-1}\vphi_{\vx},
\end{equation}
where
\begin{equation}
   \MS_{\activeset} := \tfrac{1}{\sigma^{2}}\sum_{\vx'\in \activeset}\vphi_{\vx'}\vphi_{\vx'}^{\!\top} + \MI.
\end{equation}
so all dot-products and inversions are only in a $K$-dimensional space with $K\!\ll\!P$. 
Here $\MS_{\activeset}$ is the posterior precision matrix of a Bayesian linear regressor built from the projected Jacobians $\vphi_{\vx}$ on $\activeset$; consequently, removing a single example in subsequent greedy steps corresponds to a rank-one update, for which Sherman-Morrison formula yields efficient $O(K^2)$ maintenance of $\MS_{\activeset}^{-1}$.
This addresses the memory consideration (no need to retain full $P$-dimensional Jacobians) and maintains the reasonable compute (cheap solves in the sketch space). 
Practically, the projected Jacobians $\vphi_{\vx}$ can be obtained without materializing full Jacobians, an approach advocated in recent scalable TDA pipelines \citep{choe2024your,schioppa2024efficient}.

\paragraph{Vector database.}
Even under the InfoGain relaxation, each greedy step in \cref{eq:info_gain_greedy_rank1} still requires scanning all candidates in $\data \setminus \activeset$.
Approximate nearest-neighbour (ANN) indices in high-performance vector databases (e.g., FAISS \citep{johnson2019billion}) can substantially reduce this cost by turning the search into vector similarity queries.
This perspective has recently been drawn on to scale TDA to LLM regimes \citep{sun2025enhancing,liu-etal-2025-olmotrace}.
However, such system optimizations only apply to primitives that reduce to vector similarity search.
While a full database implementation is beyond our scope, we use this as a design constraint.
In \cref{eq:info_gain_greedy_rank1}, the numerator $\smash{\bigl(v_{\vx,*}^{\activeset}\bigr)^2}$ is amenable to this setting,\footnote{FAISS \citep{johnson2019billion} supports an absolute inner-product metric for some indexes, which preserves the ranking induced by a squared inner product. Otherwise, a squared inner product can be emulated via two \emph{vanilla} inner-product queries.} whereas the denominator $\smash{v_{\vx}^{\activeset}}$ is a candidate-specific quadratic form and is not directly expressible as a standard vector-search primitive.

\begin{figure}[t]
   \centering
   \begin{subfigure}[b]{\columnwidth}
      \centering
      \includegraphics[width=\columnwidth]{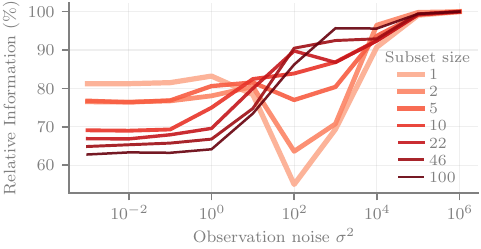}
      \caption{Information efficiency vs. observation noise}
      \label{fig:info_efficiency_obsnoise}
   \end{subfigure}
   
   \vspace{1em}
   
   \begin{subfigure}[b]{\columnwidth}
      \centering
      \includegraphics[width=\columnwidth]{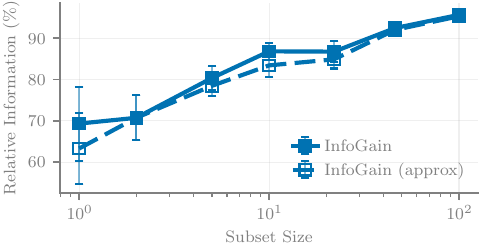}
      \caption{Information efficiency vs. subset size}
      \label{fig:info_efficiency_subsetsize}
   \end{subfigure}
   \caption{
   Information efficiency experiments on binary CIFAR-10.
   Top: selecting subsets via greedy InfoGain and varying the observation noise. Bottom: fixed observation noise level ($\sigma^2=10^3$), using InfoGain and approximate InfoGain.}
   \label{fig:info_efficiency}
\end{figure}

\paragraph{Linear-response variance correction.}
In order to have an algorithm that can be implemented as queries to a vector database, we now derive a further approximation (see App.~\ref{app:linear_repsonse} for further details). Specifically, we derive a first-order (linear-response) perturbation of the test marginal variance in \cref{eq:info_gain_greedy} by introducing a scalar weight on the candidate point and differentiating with respect to it, in the spirit of \citet{giordano2018covariances} (cf. \citealp{opper2003variational,welling2004linear}):
\begin{equation}
v_*^{\activeset \cup \{\vx\}} \approx v_*^{\activeset} - \tfrac{1}{\sigma^{2}} \bigl(v_{\vx,*}^{\activeset}\bigr)^2.
\end{equation}
Substituting this approximation into \cref{eq:info_gain_greedy} removes the candidate-dependent denominator and yields the following greedy rule for the \emph{variance-corrected} information gain,
\begin{equation}
   \vx^{(*)} \;\leftarrow\; \argmax_{\vx \in \data \setminus \activeset} \bigl(\vphi_{\vx}^{\!\top}\vr_{*}^{\activeset}\bigr)^2,
   \label{eq-obj-ig-approx}
\end{equation}
where
\[
   \vr_{*}^{\activeset} = \MS_{\activeset}^{-1}\,\vphi_{*} = \vphi_* \;-\; \MPhi_{\activeset}^{\!\top}\!\bigl(\MPhi_{\activeset}\MPhi_{\activeset}^{\!\top} + \sigma^2 \MI\bigr)^{-1}\!\MPhi_{\activeset}\,\vphi_* .
\]
Here $\MPhi_{\activeset}$ stacks the projected Jacobians for $\activeset$ as an $(M\times K)$ matrix.
The second line shows that $\vr_{*}^{\activeset}$ is a ``residual’’ query obtained by removing from $\vphi_*$ the energy explained by the acquired set (a kernel-ridge projection).
Operationally, each greedy step reduces to a squared inner-product search over the fixed database $\{\vphi_{\vx}\}_{\vx\in\data\setminus\activeset}$ with the single evolving query $\vr_*^{\activeset}$—exactly the primitive supported by vector databases.

\paragraph{Approximation Quality.}
To evaluate how well InfoGain and approximate InfoGain with the linear-response correction track InfoLoss, we introduce a metric which we call relative information (Figure \ref{fig:info_efficiency}). Relative information tracks the fraction of information (measured in nats) recovered by approximations, relative to the information loss criterion (approximated by greedy selection for tractability purposes). We evaluate this on binary CIFAR-10 (class 0 vs. class 1) trained on a 3-layer MLP. Varying the observation noise (top row of Figure \ref{fig:info_efficiency}), InfoGain approaches InfoLoss across subset sizes, converging near 100\% with $\sigma^2 \geq 10^4$. This is consistent with the high-noise asymptotic equivalence in Lemma \ref{lem:high_noise_equiv}. At a fixed noise level ($\sigma^2 = 10^3$, bottom row of Figure \ref{fig:info_efficiency}), the approximate InfoGain closely tracks exact InfoGain across subset budgets, indicating that the variance-correction introduces minimal degradation while enabling vector-database-friendly selection.

\paragraph{Computational footprint.}
This approximate InfoGain pipeline requires one extra pass over the training set to compute a $K$-dimensional projected Jacobian for each of the $N$ training examples, which has the same order of cost as one training epoch and is comparable to influence-based pipelines that likewise rely on a preprocessing pass to compute projected features or approximate curvature factors. Storing these features costs $O(NK)$ memory. For a query and subset budget $M$, maintaining the sketch-space precision matrix and residual via rank-one updates costs $O(MK^2)$ and is independent of $N$; the only remaining $N$-dependent operation is candidate retrieval. In the most naive implementation this retrieval is $O(NK)$ per greedy step, whereas ANN indices can avoid full scans and make the search sublinear in $N$.

\paragraph{Summary of methods and practical guidance.} We have developed three complementary instantiations of our information-theoretic criterion, which differ in how they balance faithfulness to the leave-subset-out counterfactual objective, greedy optimization guarantees and scalability. Greedy InfoLoss most directly matches the counterfactual notion of entropy increase when subsets are withheld, but it is not submodular and therefore lacks greedy optimization guarantees. Greedy InfoGain is a principled high-noise relaxation whose leading-order term matches InfoLoss, while enabling submodular optimization with the standard $(1-1/e)$ guarantee. Finally, approximate greedy InfoGain adds the linear-response variance correction, reducing each step to a squared inner-product query against a vector database. In practice, this makes approximate greedy InfoGain approx the default choice: \cref{fig:info_efficiency} shows that it closely tracks exact greedy InfoGain, while being the only variant that maps directly to standard vector-database primitives and thus cleanly scales to very large candidate pools. Exact greedy InfoGain is preferable when the candidate set is moderate and one wants the exact submodular objective, whereas InfoLoss is most appropriate when staying as close as possible to the leave-subset-out counterfactual definition matters more than greedy optimization guarantees or retrieval efficiency.

\section{EXPERIMENTS}\label{sec:experiments}
We evaluate our Bayesian information-theoretic attribution methods---greedy\footnote{While information loss is not sub-modular (unlike information gain), we can still optimize it greedily.} \emph{InfoLoss}, greedy \emph{InfoGain}, and approximate greedy \emph{InfoGain} (implementable with a vector database)---on three complementary tasks: (i) \emph{leave-subset-out brittleness}, which probes how well an attribution method identifies training examples whose removal most degrades performance; (ii) \emph{retrieving ground-truth attribution via a synthetic backdoor}, which provides a controlled setting with known influential training instances; and (iii) \emph{coreset selection}, which asks whether attributed subsets suffice for accurate retraining.
All experiments operate in a \emph{final-model-only} regime \citep{wei2024final} where only the final checkpoint is available and methods must avoid retraining the baseline model for scoring (for instance as done in \citet{choe2024your}). We treat the observation noise $\sigma^2$ in our attribution methods as a hyperparameter and tune this on a held-out split using a log-scaled grid from $10^{-6}$ to $10^6$ (see App.~\ref{app:exp-details} for full details).

\begin{figure*}[t]
    \centering
    \includegraphics[width=\textwidth]{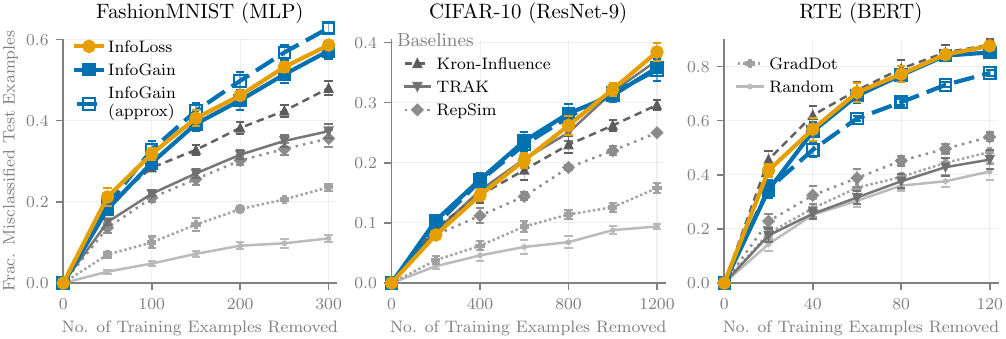}
    \caption{Our Bayesian information-theoretic methods---\textsc{InfoLoss}, \textsc{InfoGain}, and \textsc{InfoGain (approx)}---recover strong attribution signal in identified subsets, as reflected by higher brittleness. We plot the fraction of previously correct test queries that become misclassified after removing the \emph{same-label} training examples attributed by each method and retraining. Left: Fashion-MNIST (MLP), where our methods dominate and \textsc{KronInfluence} is the closest baseline. Middle: CIFAR-10 (ResNet-9), where our methods lead across budgets with \textsc{TRAK} most competitive. Right: RTE (BERT), where \textsc{KronInfluence} is strongest at smaller budgets, while \textsc{InfoLoss}/\textsc{InfoGain} catch up and are competitive at larger budgets. Error bars represent standard error across repeated runs.
    }
    \label{fig:exp_brittleness}
 \end{figure*}

\paragraph{Baselines.}
We compare against standard TDA baselines used in prior work: \textsc{Random}, representational similarity (\textsc{RepSim} / cosine similarity on penultimate-layer features) \citep{hanawaevaluation}, gradient dot-product (\textsc{GradDot} / \ie TracIn with final checkpoint only) \citep{pruthi2020estimating}, \textsc{TRAK} \citep{park2023trak}, and \textsc{KronInfluence} (influence function with EK-FAC approximation) \citep{grosse2023studying}\footnote{For the brittleness task, influence-based methods directly target loss increases upon removal, whereas our methods target information-theoretic criteria that are only indirectly related to accuracy. We nevertheless include this evaluation because it has become standard in the attribution literature.}.

\paragraph{Implementation.}
We avoid explicit multi-output Jacobians by working with a scalar measurement function.
For brittleness and backdoor retrieval, we compute training-example Jacobians with respect to the logit corresponding to the query's ground-truth class (rather than Jacobians of all logits).
For the coreset experiment, where attribution is defined with respect to a set of queries, we instead use the same multiclass reduction as \textsc{TRAK}.
All Jacobians are projected using a Rademacher sketch, shared across projection-based methods (ours and \textsc{TRAK}) within each dataset; we use projection dimension $4096$ for CIFAR-10/Fashion-MNIST and $20480$ for RTE.
We reuse the efficient CUDA Jacobian/sketch implementation from \citet{park2023trak}.
Given the sketched features, our methods solve a Bayesian linear regression in closed form; all variance computations use a numerically stable Cholesky factorization.
For the vision models in brittleness/backdoor (CIFAR-10 and Fashion-MNIST), before computing Jacobians we convert the trained network to the NTK parameterization \citep{jacot2018neural}, which preserves predictions but introduces width-dependent scaling of Jacobians; we apply this reparameterization \emph{post hoc} solely for attribution (we do not train in NTK parameterization). For the CIFAR-10 coreset experiment, we keep the original parameterization.
For BERT on RTE, we do not apply NTK parameterization and restrict Jacobians to linear layers.

\subsection{Leave-subset-out brittleness}\label{sec:brittleness}
\paragraph{Setup.}
We follow the subset-removal counterfactual protocol first introduced in \citet{singla2023simple} but the experimental setup is adapted from \citet{bae2024training}.
We evaluate on CIFAR-10 and Fashion-MNIST (the latter subsampled to 25\% of the original training set) using a ResNet-9 and MLP respectively, and on GLUE-RTE using BERT (see App.~\ref{app:exp-details} for further details).
In contrast to \citet{bae2024training}, the candidate pool for subset selection is restricted to \emph{same-class} training examples as the query, following \citet{singla2023simple}, to mitigate the bias of certain methods towards attribution in the same-class \citep{hanawaevaluation}.

\paragraph{Evaluation protocol.}
Here we describe the details of our evaluation protocol:
\begin{enumerate}[label = (\arabic*)]
  \item Train a base model on the full training set.
  \item From the test split, identify correctly classified examples as queries ($100$ for CIFAR-10/Fashion-MNIST and $50$ for RTE).
  \item For each query: 
  \begin{enumerate}
    \item Perform attribution according to each method (for a fixed subset size).
      \item Remove the attributed subset and retrain from scratch (performed for the number of query points).
      \item Predict on the query point and flag if misclassification occurs. 
  \end{enumerate}

  \item Finally we report the proportion of all query points that are misclassified. This is repeated for a linearly-spaced grid of subset sizes (if a query point becomes misclassified for a given subset size then this is terminated early).
\end{enumerate}

\begin{figure}[t]
    \centering
    \includegraphics[width=\columnwidth]{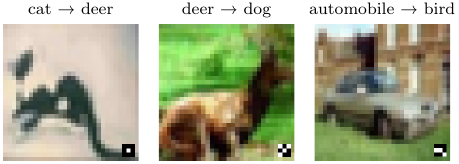}
    \caption{Examples of backdoored CIFAR-10 images with triggers. The trigger is a $3{\times}3$ black/white random pattern placed at the bottom-right corner. This causes the model to misclassify the image to the corrupted class.}
    \label{fig:backdoor_examples}
\end{figure}

\paragraph{Results.}
In \cref{fig:exp_brittleness}, we observe that across all three datasets our Bayesian information-theoretic methods perform strongly on the brittleness metric. On Fashion-MNIST (left), \textsc{InfoLoss}, \textsc{InfoGain}, and \textsc{InfoGain (approx)} all induce substantially more counterfactual errors than the baselines at every removal budget, with the gap widening as the subset size grows; \textsc{InfoGain (approx)} is typically the top curve and closely tracks (or slightly exceeds) the exact \textsc{InfoGain}. Among baselines, \textsc{KronInfluence} is the most competitive on Fashion-MNIST. On CIFAR-10 (middle), our methods again lead throughout; \textsc{InfoGain} is strongest at small--mid budgets while \textsc{InfoLoss} edges out others at the largest budget. Here, however, \textsc{TRAK} performs very close to our methods, with curves that mostly overlap. On RTE with BERT (right), the strongest baseline (\textsc{KronInfluence}) is highly competitive and tends to lead at smaller budgets, while \textsc{InfoLoss} and \textsc{InfoGain} close the gap and are competitive at larger budgets; \textsc{InfoGain (approx)} remains weaker than its exact counterpart but still outperforms the other baselines.
Despite influence-derived approaches (\ie \textsc{TRAK} and \textsc{KronInfluence}) being tailored to loss increases, our information-centric objectives remain competitive and are often superior, indicating that the information captured by our scores translates into substantial counterfactual sensitivity under subset removal (although the same-class candidate restriction may also amplify this effect).

\subsection{Retrieving ground-truth attribution via backdoor}\label{sec:backdoor}

While brittleness measures whether the subsets identified by an attribution method matter counterfactually, it does not tell us whether the retrieved examples are truly the ones driving the prediction.
To test that more directly, we next turn to a synthetic backdoor setting where the relevant training instances are known by construction.

\begin{table}[t]
    \centering
    \small
    \setlength{\tabcolsep}{2.5pt}
   \begin{tabular}{
     l
     S[table-format=1.3] c >{\scriptsize}S[table-format=1.3]
     S[table-format=1.3] c >{\scriptsize}S[table-format=1.3]
   }
   \toprule
  \textbf{Method} &
  \multicolumn{3}{c}{\textbf{Recall@50 ($\uparrow$)}} &
  \multicolumn{3}{c}{\textbf{MRR ($\uparrow$)}} \\
  \cmidrule{1-7}
   Random        & 0.011 & $\pm$ & 0.003 & 0.045 & $\pm$ & 0.016 \\
   GradDot       & 0.010 & $\pm$ & 0.002 & 0.217 & $\pm$ & 0.033 \\
   RepSim        & \underline{0.985} & $\pm$ & 0.004 & \textbf{\underline{1.000}} & $\pm$ & 0.000 \\
   TRAK          & 0.009 & $\pm$ & 0.002 & 0.189 & $\pm$ & 0.036 \\
   KronInfluence & 0.058 & $\pm$ & 0.011 & 0.504 & $\pm$ & 0.034 \\
   \addlinespace[2pt]
   \rowcolor{gray!10} InfoLoss          & \textbf{\underline{1.000}} & $\pm$ & 0.000 & \textbf{\underline{0.999}} & $\pm$ & 0.001 \\
   \rowcolor{gray!10} InfoGain          & 0.974 & $\pm$ & 0.007 & \textbf{\underline{0.998}} & $\pm$ & 0.001 \\
   \rowcolor{gray!10} InfoGain (approx) & 0.974 & $\pm$ & 0.007 & \textbf{\underline{0.998}} & $\pm$ & 0.001 \\
    \bottomrule
    \end{tabular}
    \caption{Our Bayesian information-theoretic method \textsc{InfoLoss} retrieves the ground-truth backdoor examples almost perfectly on CIFAR-10—achieving the highest recall and matching the best baseline on MRR (no significant difference). 
    Moreover, \textsc{InfoGain} and \textsc{InfoGain (approx)} show no significant difference from the best methods on MRR, and their recall is close to the best baseline.
    We compare standard TDA baselines to our proposed methods (shaded). We denote the best within each block (Baselines / Proposed) with underline and bold denotes the best overall. Reported metrics are accompanied by standard error from repeated runs.}
    \label{tab:exp_backdoor}
 \end{table}

\paragraph{Backdoor design.}
We use CIFAR-10 with a class-conditional (cyclic) BadNets \citep{gu2019badnets} rule: for base class $k$, a backdoored image is relabelled to $(k{+}1)\bmod 10$.
The backdoor trigger is a $3{\times}3$ black/white randomly-generated grid placed at the bottom-right corner, that is unique to each class (see \cref{fig:backdoor_examples}).
This is applied to 1\% of the training set, distributed evenly across classes (50 per class).
With this configuration, the attack success rate on the poisoned test set is near-perfect whilst the clean test accuracy is unchanged.

\paragraph{Retrieval task and metrics.}
For each backdoored test query from base class $k$, the ground-truth attribution set are those poisoned training examples whose base class is $k$.
We report the standard IR metrics, \emph{recall@50} and mean reciprocal rank (\emph{MRR}), averaged over $100$ queries (see App.~\ref{app:exp-details} for further details).
All other aspects of the experimental setup are identical to the brittleness experiment on CIFAR-10.

\paragraph{Results.} In \cref{tab:exp_backdoor}, \textsc{InfoLoss} attains near-perfect retrieval and yields significantly higher recall than all baselines in this synthetic backdoor-retrieval task. For ranking quality (MRR), the top methods—\textsc{RepSim}, \textsc{InfoLoss}, and both \textsc{InfoGain} variants—are statistically indistinguishable. The \textsc{InfoGain} and \textsc{InfoGain (approx)} methods also deliver performance close to the strongest baseline while maintaining MRR on par with the best methods. In contrast, influence-based TDA baselines struggle to retrieve the true influential training points, underscoring the advantage of our Bayesian information-theoretic approach for this retrieval setting.

\begin{figure}[t]
    \centering
    \includegraphics[width=\linewidth]{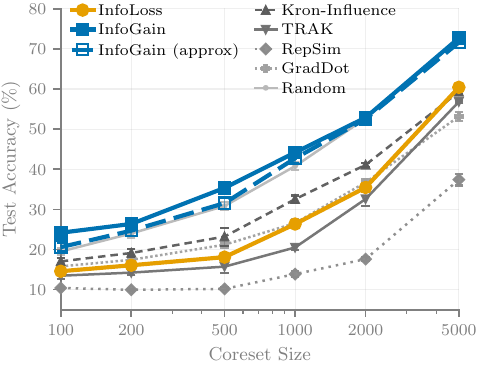}
    \caption{\textsc{InfoGain} and \textsc{InfoGain (approx)} construct substantially better CIFAR-10 coresets than existing TDA baselines which perform worse than random selection. We plot test accuracy after retraining on the selected subset for coreset sizes ranging from $0.2\%$ to $10\%$ of the training set. Error bars show standard error across repeated runs.}
    \label{fig:coreset_cifar10}
\end{figure}

\subsection{Coreset selection}\label{sec:coreset}

We next consider coreset selection through the lens of TDA: identifying a small training subset that still supports accurate retraining.

\paragraph{Setup.}
We consider CIFAR-10 with ResNet-9 and coresets ranging from $0.2\%$ to $10\%$ of the training set.
We extract $500$ examples from the test set as the query set for attribution, and report test accuracy on the remaining held-out test examples after retraining on the selected coreset.
For the TDA baselines (\textsc{TRAK}, \textsc{KronInfluence}, \textsc{GradDot}, and \textsc{RepSim}), which do not explicitly define attribution to multiple query points at once, we average the per-query attribution scores before selecting the top-$M$ examples. This is consistent with the usual additivity assumption in influence-based TDA and with the linear datamodelling score used to evaluate such methods. App.~\ref{app:coreset-details} gives the full protocol and derives the tractable multi-query greedy selection criterion used by our methods.

\paragraph{Results.}
\Cref{fig:coreset_cifar10} shows that \textsc{InfoGain} is the strongest method across essentially the entire budget range, with the clearest gains at small coresets.
At $M=500$ (1\% of the training set), \textsc{InfoGain} reaches $35.3\%$ accuracy, compared with $30.8\%$ for \textsc{Random}, $23.2\%$ for \textsc{KronInfluence}, and $15.7\%$ for \textsc{TRAK}; at $M=1000$, the gap over \textsc{Random} remains substantial ($44.2\%$ vs.\ $40.8\%$).
\textsc{InfoGain (approx)} closely tracks the exact greedy variant, especially at larger budgets.
In contrast, \textsc{InfoLoss} is markedly weaker on this task, which is consistent with its leave-subset-out objective being tailored to identifying removals that maximally disrupt a prediction rather than subsets that are jointly informative for many queries.

The TDA baselines all underperform \textsc{Random}, indicating that influence-based rankings do not cleanly transfer to coreset construction.
At the same time, a comparison with class-balanced variants (see \cref{app:exp-details}) shows that much of this degradation is due to imbalance in the selected subsets: enforcing per-class quotas substantially improves all TDA baselines and often lifts the strongest ones above greedy \textsc{InfoLoss}.
Nevertheless, even after this correction they still remain below both \textsc{InfoGain} variants.

\section{CONCLUSION}
We proposed a Bayesian information-theoretic approach to training data attribution (TDA), framing attribution as the information loss incurred when subsets of training data are withheld. Using Gaussian process surrogates from neural tangent features, we derived closed-form estimators and scalable relaxations to information gain, enabling efficient retrieval with Jacobian sketches and vector database queries. Empirical results on brittleness, backdoor retrieval and coreset selection show that our methods reliably identify influential subsets, recover ground-truth poisoned data and construct small training sets that preserve downstream accuracy, while maintaining computational efficiency comparable to influence-based baselines. This work bridges principled information measures with scalable attribution, offering a foundation for future applications in auditing, safety, and interpretability.

\subsubsection*{Acknowledgements}
DT would like to thank Alice Wang and the other members of the Simplex Lab at Spotify. DT also acknowledges Emtiyaz Khan and Eric Nalisnick for discussions around sensitivity-based variance estimates.

\bibliography{main}
\bibliographystyle{plainnat}

\clearpage
\appendix
\thispagestyle{empty}

\onecolumn
\aistatstitle{A Bayesian Information-Theoretic Approach to Data Attribution \\
-- Supplementary Materials --}

\section{HIGH-NOISE ASYMPTOTIC EQUIVALENCE BETWEEN INFORMATION LOSS AND INFORMATION GAIN}\label{app:info_criteria_obs_noise}
The goal of this section is to prove the following Lemma:
\highnoisequiv*
\textit{Proof.} Fix the subset size $M$ and define $\mathcal{S}_M := \{S \subseteq \data : |S| = M\}$. For notational convenience, write $a(S) := \norm{\vk_{S*}}^2$, use $v_*(\cdot)$ for the marginal variance of $f_* = f(\vx_*)$ as a function of the conditioning set, and define $\bar{S} = \data \setminus S$ as the complement of the training subset. Then we can write the criteria as,
\begin{align}
   \mathcal{I}_{\mathrm{IG}}(S;\vx_*)
        &= -\frac{1}{2} \log \frac{v_*(S)}{k_{**}},
        \label{eq:crit1} \\
   \mathcal{I}_{\mathrm{IL}}(S;\vx_*)
        &= \frac{1}{2} \log \frac{v_*(\bar{S})}{v_*(\data)}.
        \label{eq:crit2}
\end{align}
where we consider the marginal variances $v_*(S)$ and $v_*(\bar{S})$ under an underlying GP model without prescribing the kernel further.
Using the Neumann series $(\MI + \MA)^{-1} = \sum_{k=0}^{\infty} (-1)^k \MA^k$ for which convergence is assured when $\rho(\MA) < 1$ where $\rho(\cdot)$ is the spectral radius, we can show:
\begin{align}
    \rnd{\MK + \sigma^2 \MI}^{-1} &= \sigma^{-2} \rnd{\MI + \sigma^{-2} \MK}^{-1} \\
    &= \sigma^{-2} \rnd{ \MI - \sigma^{-2} \MK + \sigma^{-4} \MK^2 + \ldots } \\
    &= \sigma^{-2} \MI + \bigo(\sigma^{-4}) \label{eq:neumann}
\end{align}
where the remainder is entrywise of order $\sigma^{-4}$ and the expansion is valid whenever $\rho(\sigma^{-2}\MK) < 1$, which holds for sufficiently large $\sigma^2$. Because $\mathcal{S}_M$ is finite, all $\bigo(\sigma^{-4})$ terms below can be taken uniformly over $S \in \mathcal{S}_M$.
Using \cref{eq:neumann}, we can write the marginal variance $v_*(S)$ as,
\begin{align}
    v_*(S) &= k_{**} - \vk_{S*}^\T \rnd{\MK_{S} + \sigma^2 \MI_M}^{-1} \vk_{S*} \\
    &= k_{**} - \sigma^{-2} a(S) + \bigo(\sigma^{-4}). \label{eq:var_s_simplified}
\end{align}
Hence,
\begin{align}
    \mathcal{I}_{\mathrm{IG}}(S;\vx_*)
    &= -\frac{1}{2} \log \rnd{1 - \frac{a(S)}{k_{**}\sigma^2} + \bigo(\sigma^{-4})} \\
    &= \frac{a(S)}{2k_{**}\sigma^2} + \bigo(\sigma^{-4}), \label{eq:ig_asymptotic}
\end{align}
where the second line uses the Taylor expansion $\log(1+u) = u + \bigo(u^2)$ together with $u = \bigo(\sigma^{-2})$.
For $v_*(\bar{S})$, we first turn to the formula for the leave-$M$-out marginal variance (see derivation in \cref{sec:derivation_press}),
\begin{equation}\label{eq:press}
    v_*(\bar{S}) = v_*(\data) + [ \MK_y^{-1} \vk_{\data*}]_S^\T \rnd{[\MK_y^{-1}]_{S,S}}^{-1} [ \MK_y^{-1} \vk_{\data*} ]_S
\end{equation}
where we denote $\MK_y := \MK_{\data\data} + \sigma^2 \MI_N$ and overload $S$ as an index in order to extract entries corresponding to rows/columns of $S$ in $\data$.
Again using \cref{eq:neumann}, we can expand the two factors on the right-hand side of \cref{eq:press}. First,
\begin{align}
    [\MK_y^{-1}]_{S,S}
    &= \sigma^{-2} \MI_M - \sigma^{-4} \MK_{SS} + \bigo(\sigma^{-6}) \\
    \implies \rnd{[\MK_y^{-1}]_{S,S}}^{-1}
    &= \sigma^{2} \rnd{\MI_M - \sigma^{-2} \MK_{SS} + \bigo(\sigma^{-4})}^{-1} \\
    &= \sigma^{2} \MI_M + \MK_{SS} + \bigo(\sigma^{-2}). \label{eq:MKy_inv_S_S_inv}
\end{align}
and second,
\begin{align}
    [ \MK_y^{-1} \vk_{\data*}]_S
    &= \sigma^{-2} \sqr{\MI_N - \sigma^{-2} \MK_{\data\data} + \bigo(\sigma^{-4})}_{S,:} \vk_{\data*} \\
    &= \sigma^{-2} \rnd{\vk_{S*} - \sigma^{-2} \MK_{S\data} \vk_{\data*} + \bigo(\sigma^{-4})} \\
    &= \sigma^{-2} \vk_{S*} + \bigo(\sigma^{-4}). \label{eq:MKy_inv_vk_T*_S}
\end{align}
Substituting \cref{eq:MKy_inv_S_S_inv,eq:MKy_inv_vk_T*_S} into \cref{eq:press} we obtain,
\begin{align}
    v_*(\bar{S}) - v_*(\data)
    &= \rnd{\sigma^{-2} \vk_{S*} + \bigo(\sigma^{-4})}^\T
    \rnd{\sigma^{2} \MI_M + \MK_{SS} + \bigo(\sigma^{-2})}
    \rnd{\sigma^{-2} \vk_{S*} + \bigo(\sigma^{-4})} \\
    &= \sigma^{-2} \vk_{S*}^\T \vk_{S*} + \bigo(\sigma^{-4}) \\
    &= \sigma^{-2} a(S) + \bigo(\sigma^{-4}). \label{eq:var_s_bar_simplified}
\end{align}
Moreover, applying \cref{eq:var_s_simplified} with $S = \data$ yields
\begin{equation}
    v_*(\data) = k_{**} + \bigo(\sigma^{-2}). \label{eq:v_full_asymptotic}
\end{equation}
Therefore,
\begin{align}
    \mathcal{I}_{\mathrm{IL}}(S;\vx_*)
    &= \frac{1}{2} \log \rnd{1 + \frac{v_*(\bar{S}) - v_*(\data)}{v_*(\data)}} \\
    &= \frac{1}{2} \log \rnd{1 + \frac{a(S)}{k_{**}\sigma^2} + \bigo(\sigma^{-4})} \\
    &= \frac{a(S)}{2k_{**}\sigma^2} + \bigo(\sigma^{-4}), \label{eq:il_asymptotic}
\end{align}
where in the second line we used \cref{eq:var_s_bar_simplified,eq:v_full_asymptotic}.
\cref{eq:ig_asymptotic,eq:il_asymptotic} prove the shared leading-order expansion.

For the final claim, suppose $a(S)$ has a unique maximizer $S^\dagger$ over $\mathcal{S}_M$ and define the gap
\begin{equation}
    \Delta := a(S^\dagger) - \max_{S \in \mathcal{S}_M,\; S \neq S^\dagger} a(S) > 0.
\end{equation}
Since $k_{**} > 0$, by uniformity of \cref{eq:ig_asymptotic,eq:il_asymptotic} there exist constants $C > 0$ and $\sigma_{\mathrm{asym}}^2 > 0$ such that for both criteria $\mathcal{J} \in \{\mathcal{I}_{\mathrm{IG}}, \mathcal{I}_{\mathrm{IL}}\}$ we may write,
\begin{equation}
    \mathcal{J}(S;\vx_*)
    =
    \frac{a(S)}{2k_{**}\sigma^2}
    +
    r_{\mathcal{J},\sigma}(S),
    \qquad
    \left| r_{\mathcal{J},\sigma}(S) \right| \leq C \sigma^{-4}
\end{equation}
for all $S \in \mathcal{S}_M$ and $\sigma^2 \geq \sigma_{\mathrm{asym}}^2$.
Now fix any $S \neq S^\dagger$.
Then for either criterion $\mathcal{J}$,
\begin{align}
    \mathcal{J}(S^\dagger;\vx_*) - \mathcal{J}(S;\vx_*)
    &= \frac{a(S^\dagger) - a(S)}{2k_{**}\sigma^2}
    + r_{\mathcal{J},\sigma}(S^\dagger)
    - r_{\mathcal{J},\sigma}(S) \\
    &\geq \frac{a(S^\dagger) - a(S)}{2k_{**}\sigma^2} - 2C\sigma^{-4} \\
    &\geq \frac{\Delta}{2k_{**}\sigma^2} - 2C\sigma^{-4}
    \label{eq:high_noise_gap_bound}
\end{align}
The lower bound in \cref{eq:high_noise_gap_bound} is positive whenever
\begin{equation}
    \frac{\Delta}{2k_{**}\sigma^2} - 2C\sigma^{-4} > 0
    \qquad \Longleftrightarrow \qquad
    \sigma^2 > \frac{4Ck_{**}}{\Delta}.
\end{equation}
Therefore, if
\begin{equation}
    \sigma^2 > \sigma_{\mathrm{crit}}^2
    :=
    \max\left\{
        \sigma_{\mathrm{asym}}^2,\,
        \frac{4Ck_{**}}{\Delta}
    \right\},
\end{equation}
then $\mathcal{J}(S^\dagger;\vx_*) - \mathcal{J}(S;\vx_*) > 0$ for every $S \neq S^\dagger$ and either criterion $\mathcal{J}$.
Hence both $\mathcal{I}_{\mathrm{IG}}(\cdot;\vx_*)$ and $\mathcal{I}_{\mathrm{IL}}(\cdot;\vx_*)$ are uniquely maximized by $S^\dagger$ for all $\sigma^2 > \sigma_{\mathrm{crit}}^2$. This proves the lemma.

\subsection{Derivation of leave-$M$-out marginal variance}\label{sec:derivation_press}
This is a generalization of the leave-one-out expression commonly used in cross-validation of GPs \citep{sundararajan1999predictive} (also see Ch.~5.4.2 in \citet{williams2006gaussian}) but in addition to leaving out multiple points, here we are concerned with evaluating variance at a test point rather than the omitted point(s).
The derivation proceeds by exploiting the partitioned matrix inversion lemma to relate the marginal variance on the full data $v_*(\data)$ to the leave-$M$-out marginal variance $v_*(\bar{S})$.
To recap the marginal variance on the full data is given by,
\begin{equation}
    v_*(\data) = k_{**} - \vk_{\data*}^\T \MK_y^{-1} \vk_{\data*}, \quad \text{where} \; \MK_y := \MK_{\data\data} + \sigma^2 \MI_N. \label{eq:var_T_full}
\end{equation}

We consider a partition of $\MK_y$ into the following block matrix with the first block corresponding to the retained points $\bar{S}$ and the second block corresponding to the omitted points $S$ (the expression for marginal variance allows for arbritrary ordering of the points so we can re-order the points as necessary):
\begin{align}
    \MK_y &= \begin{bmatrix}
        \MA & \MB \\
        \MB^\T & \MC
    \end{bmatrix}
    \quad\text{with}\quad
    \begin{aligned}
        \MA &= \MK_{\bar{S}\bar{S}} + \sigma^2 \MI_{N-M} \\
        \MB &= \MK_{\bar{S}S} \\
        \MC &= \MK_{SS} + \sigma^2 \MI_M
    \end{aligned}
\end{align}
Then using the partitioned matrix inversion lemma, we have,
\begin{equation}\label{eq:MKy_inv}
    \MK_y^{-1} = \begin{bmatrix}
        \MA^{-1} + \MA^{-1} \MB \MM \MB^\T \MA^{-1} & - \MA^{-1} \MB \MM \\
        - \MM \MB^\T \MA^{-1} & \MM
    \end{bmatrix}
\end{equation}
where $\MM^{-1} = \MC - \MB^\T \MA^{-1} \MB$ is the Schur complement of $\MA$ in $\MK_y$.
Substituting \cref{eq:MKy_inv} into the expression for $v_*(\data)$ given in \cref{eq:var_T_full} and expanding $\vk_{\data*}^\T = [ \vk_{\bar{S}*} \; \vk_{S*} ]^\T$ we have,
\begin{align}
    v_*(\data) &= k_{**} - \vk_{\bar{S}*}^\T \MA^{-1} \vk_{\bar{S}*} -
    (\vk_{S*} - \MB^\T \MA^{-1} \vk_{\bar{S}*})^\T \MM (\vk_{S*} - \MB^\T \MA^{-1} \vk_{\bar{S}*}) \\
    &= v_*(\bar{S}) - (\vk_{S*} - \MB^\T \MA^{-1} \vk_{\bar{S}*})^\T \MM (\vk_{S*} - \MB^\T \MA^{-1} \vk_{\bar{S}*}) \label{eq:var_T_quad}
\end{align}
where we noticed that the first two terms on the right-hand side of the first line are exactly $v_*(\bar{S})$.
Next, we consider the following,
\begin{align}
    &[\MK_y^{-1}]_{S,\data} \, \vk_{\data*} = \begin{bmatrix} -\MM \MB^\T \MA^{-1} & \MM \end{bmatrix} \begin{bmatrix} \vk_{\bar{S}*}\\\vk_{S*} \end{bmatrix}
    = \MM (\vk_{S*} - \MB^\T \MA^{-1} \vk_{\bar{S}*}) \\
    &\qquad \implies \vk_{S*} - \MB^\T \MA^{-1} \vk_{\bar{S}*} = \MM^{-1} [\MK_y^{-1}]_{S,\data} \, \vk_{\data*} \label{eq:MKy_inv_S_T}
\end{align}
where the expression to the left of the equality in \cref{eq:MKy_inv_S_T} is identical to that in the quadratic form in \cref{eq:var_T_quad}.
We proceed to substitute \cref{eq:MKy_inv_S_T} into \cref{eq:var_T_quad} and rearrange it in terms of $v_*(\bar{S})$:
\begin{align}
    v_*(\bar{S}) &= v_*(\data) + ([\MK_y^{-1}]_{S,\data} \, \vk_{\data*})^\T \MM^{-1} [\MK_y^{-1}]_{S,\data} \, \vk_{\data*} \\
    &= v_*(\data) + [\MK_y^{-1} \vk_{\data*}]_S^\T \rnd{[\MK_y^{-1}]_{S,S}}^{-1} [\MK_y^{-1} \vk_{\data*}]_S
\end{align}
where in the second line we used a simple manipulation to write $[\MK_y^{-1}]_{S,\data} \, \vk_{\data*}$ as $[\MK_y^{-1} \vk_{\data*}]_S$ and write $\MM$ as a subblock of $\MK_y^{-1}$ (see \cref{eq:MKy_inv}).

\section{DERIVATION OF THE GREEDY RULES UNDER GP SURROGATE}\label{app:greedy_info}
\subsection{Information Gain}\label{app:greedy_info_gain}
At each of the $M$ rounds we pick the example whose inclusion most reduces the marginal variance at $\vx_*$ conditioned on the set of acquired points so far $\activeset$ and the candidate point $\vx$.
This is given in \cref{eq:info_gain_greedy} which is repeated here for convenience:
\begin{align}
   &\vx^{(m)} \leftarrow \argmin_{\vx \in \data \setminus \activeset} v_{*}^{\activeset \cup \{\vx\}} \nonumber \\
   &\text{where} \quad v_{*}^{\activeset \cup \{\vx\}} = k_{**} - \vk_{\activeset \cup \{\vx\}, *}^\T \rnd{ \MK_{\activeset \cup \{ \vx\}} + \sigma^2 \MI }^{-1} \vk_{\activeset \cup \{\vx\}, *}. \label{eq:var_aug_star}
\end{align}
Similar to \cref{sec:derivation_press}, we exploit the partitioned matrix inversion lemma to relate the ``add-one-in'' marginal variance $v_*^{\activeset \cup \{\vx\}}$ to $v_*^{\activeset}$. We consider a partition of $\MK_{\activeset \cup \{\vx\}} + \sigma^2\MI$ into a block matrix with the first block corresponding to the set of acquired points $\activeset$ and the second block corresponding to the candidate point $\vx$:
\begin{equation}
   \MK_{\activeset \cup \{\vx\}} + \sigma^2\MI = \begin{bmatrix}
        \MA & \vb \\
        \vb^\T & c
    \end{bmatrix}
    \quad\text{with}\quad
    \begin{aligned}
        \MA &= \MK_{\activeset} + \sigma^2 \MI \\
        \vb &= \vk_{\activeset, \vx} \\
        c &= k_{\vx,\vx} + \sigma^2
    \end{aligned}
\end{equation}
Using the partitioned matrix inversion lemma, we have:
\begin{equation}\label{eq:K_aug_inv}
   \rnd{\MK_{\activeset \cup \{\vx\}} + \sigma^2\MI}^{-1} = \begin{bmatrix}
        \MA^{-1} + m \MA^{-1} \vb \vb^\T \MA^{-1} & -m\MA^{-1} \vb \\
        -m \vb^\T \MA^{-1} & m
    \end{bmatrix}
\end{equation}
where $m^{-1} = c - \vb^\T \MA^{-1} \vb$ is the Schur complement of $\MA$ in $\rnd{\MK_{\activeset \cup \{\vx\}} + \sigma^2\MI}$. Substituting \cref{eq:K_aug_inv} into \cref{eq:var_aug_star} and expanding $\vk_{\activeset \cup \{\vx\}, *}^\T = [ \vk_{\activeset, *} \; \vk_{\vx, *} ]^\T$ we have:
\begin{equation}\label{eq:var_aug_star_expanded}
   v_*^{\activeset \cup \{\vx\}} = k_{**} - \vk_{\activeset, *}^\T \MA^{-1} \vk_{\activeset, *} - m \rnd{k_{\vx,*} - \vb^\T \MA^{-1} \vk_{\activeset, *}}^2.
\end{equation}
We recognise the first two terms on the right-hand side of \cref{eq:var_aug_star_expanded} as $v_*^{\activeset}$ and simplifying we obtain:
\begin{equation}\label{eq:var_aug_full}
   v_*^{\activeset \,\cup\, \{\vx\}} \;=\; v_*^{\activeset} \,-\, \frac{\rnd{v_{\vx,*}^{\activeset}}^2}{\sigma^2 + v_{\vx}^{\activeset}}
   \quad\text{with}\quad
   \begin{aligned}
      v_{*}^{\activeset} &= k_{**} - \vk_{\activeset, *}^\T \rnd{ \MK_{\activeset} + \sigma^2 \MI }^{-1} \vk_{\activeset, *} \\
      v_{\vx,*}^{\activeset} &= k_{\vx,*} - \vk_{\activeset, \vx}^\T \rnd{ \MK_{\activeset} + \sigma^2 \MI }^{-1} \vk_{\activeset, *} \\
      v_{\vx}^{\activeset} &= k_{\vx,\vx} - \vk_{\activeset, \vx}^\T \rnd{ \MK_{\activeset} + \sigma^2 \MI }^{-1} \vk_{\activeset, \vx}.
   \end{aligned}
\end{equation}
Since $v_*^{\activeset}$ is independent of $\vx$, we can ignore it in the greedy rule thereby obtaining \cref{eq:info_gain_greedy_rank1} -- the minus sign in front of the remaining term changes the problem from a minimization to a maximization.

\subsection{Information Loss}\label{app:greedy_info_loss}
Whilst the information loss criterion is not sub-modular, it can nevertheless be optimized greedily (just without guarantees) and therefore we compare against this in our experiments. The greedy rule corresponds to choosing the candidate point whose removal maximally increases the marginal variance at $\vx_*$. For notational convenience let $\activeset^{+} = \activeset \cup \{\vx\}$ then,
\begin{align}
   &\vx^{(m)} \leftarrow \argmax_{\vx \in \data \setminus \activeset} v_{*}^{\data \setminus \activeset^{+}} \nonumber \\
   &\text{where} \quad v_{*}^{\data \setminus \activeset^{+}} = k_{**} - \vk_{\data\setminus\activeset^+, *}^\T \rnd{ \MK_{\data\setminus\activeset^+} + \sigma^2 \MI }^{-1} \vk_{\data\setminus\activeset^+, *}.
\end{align}
We can use a special case of \cref{sec:derivation_press}, that is the \emph{leave-one-out} setting, to express $v_{*}^{\data \setminus \activeset^{+}}$ with $\data\setminus\activeset$ as the conditioning set throughout. This leads to,
\begin{equation}\label{eq:var_aug_full_il}
   v_*^{\data \setminus \activeset^{+}} \;=\; v_*^{\data \setminus \activeset} \,+\, \frac{\rnd{v_{\vx,*}^{\data \setminus \activeset}}^2}{\sigma^2 - v_{\vx}^{\data \setminus \activeset}}
   \quad\text{with}\quad
   \begin{aligned}
      v_{*}^{\data \setminus \activeset} &= k_{**} - \vk_{\data\setminus\activeset, *}^\T \rnd{ \MK_{\data\setminus\activeset} + \sigma^2 \MI }^{-1} \vk_{\data\setminus\activeset, *} \\
      v_{\vx,*}^{\data \setminus \activeset} &= k_{\vx,*} - \vk_{\data\setminus\activeset, \vx}^\T \rnd{ \MK_{\data\setminus\activeset} + \sigma^2 \MI }^{-1} \vk_{\data\setminus\activeset, *} \\
      v_{\vx}^{\data \setminus \activeset} &= k_{\vx,\vx} - \vk_{\data\setminus\activeset, \vx}^\T \rnd{ \MK_{\data\setminus\activeset} + \sigma^2 \MI }^{-1} \vk_{\data\setminus\activeset, \vx}.
   \end{aligned}
\end{equation}
Thus, the greedy rule can be stated as,
\begin{equation}
   \vx^{(m)} \;\leftarrow\; \argmax_{\vx \in \data \setminus \activeset} \frac{\rnd{{v_{\vx,*}^{\data \setminus \activeset}}}^2}{\sigma^2 - v_{\vx}^{\data \setminus \activeset}}
\end{equation}

\section{DERIVATION OF LINEAR-RESPONSE VARIANCE CORRECTION}\label{app:linear_repsonse}
We consider the ``weight-space'' form (via Woodbury formula) of the marginal variance given in \cref{eq:info_gain_greedy}:
\begin{equation}
   v_*^{\activeset \cup \{\vx\}} = \vphi_*^\T \rnd{\MS_{\activeset} + \tfrac{1}{\sigma^{2}} \vphi_{\vx} \vphi_{\vx}^\T}^{-1} \vphi_*
\end{equation}
where we use $\vphi_{\vx}$ to denote the tangent features (with slight abuse of notation) and $\MS_{\activeset}$ is the posterior precision matrix.
We can construct a perturbative variant of this expression by introducing a ``weight'' $\epsilon$ on the term corresponding to the candidate point. For notational convenience let us denote $\bar{v}_* = v_*^{\activeset \cup \{\vx\}}$ then,
\begin{align}
    \bar{v}_*^\epsilon &= \vphi_*^\T \rnd{\MS_{\activeset} + \epsilon \tfrac{1}{\sigma^2} \vphi_{\vx} \vphi_{\vx}^\T}^{-1} \vphi_* \\
    &= \vphi_*^\T \rnd{ \MS_{\activeset}^{-1} - \frac{\epsilon \MS_{\activeset}^{-1} \vphi_{\vx} \vphi_{\vx}^\T \MS_{\activeset}^{-1}}{\sigma^2 + \epsilon \vphi_{\vx}^\T \MS_{\activeset}^{-1} \vphi_{\vx} }} \vphi_* \\
    &= v_*^{\activeset} - \frac{\epsilon \rnd{v_{\vx,*}^{\activeset}}^2}{\sigma^2 + \epsilon v_{\vx}^{\activeset}} \label{eq:perturbed_var}
\end{align}
where in the last line we defined $v_*^{\activeset} = \vphi_*^\T \MS_{\activeset}^{-1} \vphi_*$, $v_{\vx}^{\activeset} = \vphi_{\vx}^\T \MS_{\activeset}^{-1} \vphi_{\vx}$ and $v_{\vx,*}^{\activeset} = \vphi_{\vx}^\T \MS_{\activeset}^{-1} \vphi_{*}$.
In the second line above, we used the Sherman-Morrison formula giving a form typical of recursive least-squares or Kalman filtering.
The resulting expression can be interpreted as a weight-dependent functional that we can differentiate with respect to $\epsilon$ and evaluate at $\epsilon = 0$ giving the following sensitivity-based variance estimate,
\begin{equation}
    \frac{d \bar{v}_*^\epsilon}{d\epsilon} = -\frac{\sigma^2 \rnd{v_{\vx,*}^{\activeset}}^2}{\rnd{\sigma^2 + \epsilon v_{\vx}^{\activeset}}^2}, \qquad
    \frac{d \bar{v}_*^\epsilon}{d\epsilon} \bigg|_{\epsilon=0} = - \frac{1}{\sigma^2} \rnd{v_{\vx,*}^{\activeset}}^2
\end{equation}
Using this we can construct a first-order perturbative update to the predictive variance $\bar{v}_*^\epsilon \approx v_*^{\activeset} + d\bar{v}_*^\epsilon / d\epsilon |_{\epsilon=0} \cdot \epsilon$. Evaluating this at $\epsilon = 1$ we have,
\begin{equation}\label{eq:linear_response_variance_correction}
    \bar{v}_* = \bar{v}_*^{\epsilon=1} \approx \underbrace{v_*^{\activeset} - \frac{1}{\sigma^2} \rnd{v_{\vx,*}^{\activeset}}^2}_{\bar{v}_*^{\textrm{approx}}}
\end{equation}
This leads to the following (approximate) greedy rule,
\begin{equation}
   \vx^{(m)} \;\leftarrow\; \argmax_{\vx \in \data \setminus \activeset} \rnd{ \vphi_{\vx}^\T \MS_{\activeset}^{-1} \vphi_{*} }^2
\end{equation}

\paragraph{Alternative view of \cref{eq:linear_response_variance_correction} when in the regime of large observation noise} 
Starting from the marginal variance in \cref{eq:var_aug_full} we can show,
\begin{align}
    v_*^{\activeset \cup \{\vx\}} = v_*^{\activeset} - \frac{\rnd{v_{\vx,*}^{\activeset}}^2}{\sigma^2 + v_{\vx}^{\activeset}} &= v_*^{\activeset} - \sigma^{-2} \rnd{v_{\vx,*}^{\activeset}}^2 \rnd{ \frac{1}{1 + \sigma^{-2} v_{\vx}^{\activeset}} } \\
    &= v_*^{\activeset} - \sigma^{-2} \rnd{v_{\vx,*}^{\activeset}}^2 \rnd{ 1 - \sigma^{-2} v_{\vx}^{\activeset} + \bigo(\sigma^{-4}) } \\
    &= v_*^{\activeset} - \sigma^{-2} \rnd{v_{\vx,*}^{\activeset}}^2 + \bigo(\sigma^{-4}) \label{eq:test_var_large_noise}
\end{align}
where in the second line we used the geometric series $(1 + a)^{-1} = \sum_{k=0}^\infty (-1)^k a^k$.
Convergence is assured when $v_{\vx}^{\activeset} < \sigma^2$ where we used the fact that variance is non-negative.
Terms that are of order $\bigo(\sigma^{-4})$ or smaller are neglected similar to \cref{app:info_criteria_obs_noise}.
The resulting expression in \cref{eq:test_var_large_noise} has leading terms identical to the linear-response variance correction of \cref{eq:linear_response_variance_correction}.

\paragraph{Approximation error of the linear-response variance correction}
Let us denote $\Delta \bar{v}_*^{\textrm{exact}} = v_*^{\activeset \cup \{\vx\}} - v_*^{\activeset}$ and $\Delta \bar{v}_*^{\textrm{approx}} = \bar{v}_*^{\textrm{approx}} - v_*^{\activeset}$ then we can show the relative error is given by,
\begin{equation}
    \frac{|\Delta \bar{v}_*^{\textrm{approx}} - \Delta \bar{v}_*^{\textrm{exact}}|}{\Delta \bar{v}_*^{\textrm{exact}}} = \frac{1}{\sigma^2} v_{\vx}^{\activeset}
\end{equation}
where we dropped the absolute value since the difference is non-negative by construction.
This indicates that the quality of the approximation is determined by the marginal variance of the candidate point. When $v_{\vx}^{\activeset}$ is large then $\bar{v}_*^{\textrm{approx}}$ overestimates the actual value.

We can write the approximation error exactly as,
\begin{equation}
    v_*^{\activeset \cup \{\vx\}} - \bar{v}_*^{\textrm{approx}} = \frac{1}{\sigma^2} \rnd{v_{\vx,*}^{\activeset}}^2 \cdot \frac{v_{\vx}^{\activeset}}{\sigma^2 + v_{\vx}^{\activeset}}
\end{equation}
For the sake of exposition, we can also give this error in Lagrange's Form as follows,
\begin{equation}\label{eq:lagrange_form}
    R_2^{\epsilon=1} = \frac{1}{2} \frac{\partial^2 \bar{v}_*^{\epsilon}}{\partial \epsilon^2} \bigg|_{\epsilon=\varepsilon} \qquad \textrm{for some} \; 0 < \varepsilon < 1
\end{equation}
This second derivative is analytically given by,
\begin{equation}\label{eq:second_order_derivative}
    \frac{\partial^2 \bar{v}_*^{\epsilon}}{\partial \epsilon^2} = \frac{2 \sigma^2 \rnd{v_{\vx,*}^{\activeset}}^2 v_{\vx}^{\activeset}}{(\sigma^2 + \epsilon v_{\vx}^{\activeset})^3}
\end{equation}
Then we can bound the remainder in \cref{eq:lagrange_form} as follows,
\begin{align}
    | R_2^{\epsilon=1} | &\leq \frac{1}{2} \max_{\varepsilon \in [0,1]} \Bigg| \frac{\partial^2 \bar{v}_*^{\epsilon}}{\partial \epsilon^2} \bigg|_{\epsilon = \varepsilon} \Bigg| \\
    &= \sigma^2 \rnd{v_{\vx,*}^{\activeset}}^2 v_{\vx}^{\activeset} \underbrace{\max_{\varepsilon \in [0,1]} \sqr{ \frac{1}{(\sigma^2 + \varepsilon v_{\vx}^{\activeset})^3} }}_{= 1 / (\sigma^2)^3} \\
    &\leq \frac{1}{\sigma^4} \rnd{v_{\vx,*}^{\activeset}}^2 v_{\vx}^{\activeset}
\end{align}
where we drop the absolute value since all the terms in \cref{eq:second_order_derivative} are non-negative. In the second line, we factor out terms independent of $\varepsilon$ from the $\max$ and observe that the maximum value is attained when $\varepsilon = 0$.

\section{CONNECTIONS BETWEEN OUR INFORMATION LOSS CRITERION AND INFLUENCE-BASED TDA}\label{app:info_loss_relation_influence}
Let us consider the singleton case, \ie $S = \{(\vx,y\})$, of \cref{eq:info_loss} with the GP surrogate:
\begin{align}
   \mathcal{I}_{\mathrm{IL}}(\{(\vx,y\});\vx_*) &= \frac{1}{2} \log\rnd{ \frac{v_*^{\data \setminus \{\vx\}}}{v_*^{\data}}} 
   \shortintertext{\color{gray}\  \  $\triangleright$ substitute $v_*^{\data\setminus\{\vx\}} = v_*^{\data} + (v_{\vx,*}^{\data})^2 / (\sigma^2 - v_{\vx}^{\data})$} 
   &= \frac{1}{2} \log \rnd{ 1 + \frac{\rnd{v_{\vx,*}^{\data}}^2}{v_*^{\data}(\sigma^2 - v_{\vx}^{\data})}} \label{eq:info_loss_il_gp}
\end{align}
where in the second line we substituted the first-round score derived in \cref{eq:var_aug_full_il}.
Similar to \cref{app:linear_repsonse}, we write the variance terms in weight-space form and use $\vphi_{\vx}$ to denote tangent features, $v_*^{\data} = \vphi_*^\T \MS_{\data}^{-1} \vphi_*$, $v_{\vx,*}^{\data} = \vphi_{\vx}^\T \MS_{\data}^{-1} \vphi_*$ and $v_{\vx}^{\data} = \vphi_{\vx}^\T \MS_{\data}^{-1} \vphi_{\vx}$, where $\MS_{\data} = \frac{1}{\sigma^2} \sum_{\vx' \in \data} \vphi_{\vx'}\vphi_{\vx'}^\T + \MI$.
Now we will show that \cref{eq:info_loss_il_gp} has resemblance to a cross-influence term normalized by self-influence terms.

Let us consider (1\textsuperscript{st}-order) influence function, the canonical approach to influence-based TDA, which performs attribution by the locally approximating the counterfactual change in test loss:
\begin{equation}
   \ell_*(\hat{\vtheta}^{\setminus \vz}) - \ell_*(\hat{\vtheta}) \approx \vg_*^\T \MH_{\hat{\vtheta}}^{-1} \vg_{\vz}
\end{equation}
where we define $\vz = (\vx,y)$, denote the gradients of the per-example train and test loss by $\vg_{\vz}$ and $\vg_*$ respectively, and let $\MH_{\hat{\vtheta}}$ be the Hessian evaluated at $\hat{\vtheta}$ (assuming convergence).
Now \citet{barshan2020relatif} claims that when the dataset contains outliers or mislabelled examples, such examples are often most influential for test examples but these are poor explanatory examples for TDA. 
They reinterpret influence function from a geometric lens and propose to use cosine similarity instead to reduce the impact of ``globally'' influential examples.
This takes the form of a cross-influence term normalized by self-influence terms:
\begin{align}
   \cos(\MH_{\hat{\vtheta}}^{-\frac{1}{2}} \vg_*, \MH_{\hat{\vtheta}}^{-\frac{1}{2}} \vg_{\vz}) &= 
   \frac{\vg_*^\T \MH_{\hat{\vtheta}}^{-1} \vg_{\vz}}{\sqrt{\vg_*^\T \MH_{\hat{\vtheta}}^{-1} \vg_{*}} 
   \cdot \sqrt{\vg_{\vz}^\T \MH_{\hat{\vtheta}}^{-1} \vg_{\vz}}} 
   \shortintertext{\color{gray}\  \  $\triangleright$ expand gradients by chain rule: $\vg_* = e_* \vphi_*$ and $\vg_{\vz} = e_{\vz} \vphi_{\vx}$} 
   &= \frac{e_* \rnd{ \vphi_*^\T \MH_{\hat{\vtheta}}^{-1} \vphi_{\vz} } e_{\vz}}{\sqrt{e_* \rnd{ \vphi_*^\T \MH_{\hat{\vtheta}}^{-1} \vphi_{*}} e_*} \cdot \sqrt{e_{\vz} \rnd{ \vphi_{\vz}^\T \MH_{\hat{\vtheta}}^{-1} \vphi_{\vz}}e_{\vz}}}
   \shortintertext{\color{gray}\  \  $\triangleright$ approximate Hessian by the Gauss-Newton matrix $\MH_{\hat{\vtheta}} \approx \MS_{\data}$ and substitute variance terms} 
   &\approx \frac{ \cancel{e_*} v_{\vx,*}^{\data} \cancel{e_{\vz}}}{\cancel{e_*} \sqrt{v_*^{\data}} \cdot \cancel{e_{\vz}} \sqrt{v_{\vx}^{\data}}}
   \shortintertext{\color{gray}\  \  $\triangleright$ observe the residuals cancel}
   &= \frac{v_{\vx,*}^{\data}}{\sqrt{v_*^{\data}} \cdot \sqrt{v_{\vx}^{\data}}}
\end{align}
where $e_{\vz} = f_{\hat{\vtheta}}(\vx) - y$ and $e_{*} = f_{\hat{\vtheta}}(\vx_*) - y_*$ are the train and test residuals respectively (assuming single-output case).
The resulting expression is independent of the training label (and test label), similar to \cref{eq:info_loss_il_gp}.
An important difference however is that this score is \emph{signed} allowing one to distinguish positively-influential examples from negatively-influential examples, sometimes called ``proponents''/``opponents'' \citep{pruthi2020estimating}.
Another difference is that whilst cosine similarity downweights train examples with high (marginal) variance, our criterion does the opposite.

\section{ADDITIONAL EXPERIMENTAL DETAILS}\label{app:exp-details}
\paragraph{Datasets and models.}
We used a ResNet-9 architecture for CIFAR-10, a two-hidden-layer MLP with ReLU activations for Fashion-MNIST, and BERT for RTE.
For the vision models in brittleness/backdoor, we apply NTK reparameterization \emph{post hoc} for attribution: we do not modify the ResNet-9 default initialization and simply apply the transformation assuming standard Normal initialization; for the MLP this follows \citet{lee2020finite} where Kaiming-Normal initialization is used without any bias terms. For the CIFAR-10 coreset experiment, we instead keep the original parameterization.
For BERT on RTE, we do not apply NTK parameterization and use linear layers only for Jacobian-based attribution.
For training, the ResNet-9 was optimized using SGD with momentum (0.9), weight decay (5e-4), a batch size of 512, over 24 epochs. We employed a triangular learning rate schedule that linearly increased to a peak learning rate of 0.5 by epoch 5 and then decayed linearly to zero by the final epoch. For Fashion-MNIST, we used SGD with momentum (0.9), batch size 64, learning rate 0.03, weight decay 1e-3, and trained for 20 epochs. For RTE, we fine-tuned BERT for 3 epochs with batch size 32, learning rate $2\times 10^{-5}$, and weight decay 0.01.
The experiments were run on an internal cluster of GPUs of the following type: Tesla V100 SXM2 16 GB.

\paragraph{Baselines.} We provide additional implementation details for the TDA baseline methods adapted from \citet{bae2024training}. All baselines produce a ranking of training examples via additive pointwise scores, and attributed subsets are obtained by the rank-and-pick rule described in \cref{sec:background}.
In addition, all baselines are evaluated on the final checkpoint only, without any model retraining.
Unless otherwise stated, all reported metrics are accompanied by standard error across five random seeds.
\begin{itemize}
	\item \textsc{Random}: Selects examples uniformly at random.
	\item \textsc{RepSim}: Ranks examples by cosine similarity between penultimate-layer features of the query and training examples.
	\item \textsc{GradDot}: Ranks examples by the dot-product between query and train gradients (\ie TracIn \citep{pruthi2020estimating} at the final checkpoint only). We use standard cross-entropy loss for train gradients and the same measurement function as \textsc{TRAK} for query gradients (detailed below).
	\item \textsc{TRAK}: An extension of generalized linear model (GLM) influence \citep{pregibon1981logistic} to neural networks using random projections and the Generalized Gauss-Newton approximation. For multi-class classification (as in our setting), \textsc{TRAK} uses a measurement function that reduces multi-class outputs to a single logistic regression (see App.~E.5 in \citet{park2023trak}). We use a Rademacher sketch with projection dimension 4096 for CIFAR-10/Fashion-MNIST and 20480 for RTE (matching our methods per dataset), but do not use ensembling as originally proposed. This is implemented via the \texttt{trak} package \citep{park2023trak} with a fast CUDA extension that avoids materializing the projection matrix.
	\item \textsc{KronInfluence}: An influence function implementation using the EK-FAC approximation to the Hessian. Computation is restricted to fully-connected and convolutional layers, the layers supported by the \texttt{kroninfluence} package \citep{grosse2023studying}. We use the same measurement function as \textsc{TRAK} for query gradients and adopt the damping factor heuristic proposed in \citet{grosse2023studying}.
\end{itemize}

\paragraph{Hyperparameter tuning (observation noise).}
For each method we tune $\sigma^2$ over the log-spaced grid $\{10^{-6},\dots,10^{6}\}$ using three random repeats per candidate value.
For the brittleness experiments, we use a reduced set of 25 queries and obtain a held-out validation set by partitioning the original test set in half; we use margin change as the validation score and fix the subset size to 500 for CIFAR-10, 150 for Fashion-MNIST and 50 for RTE.
For the coreset experiment, we use the same held-out validation split and fix the coreset size to $500$.
Tuned values are reused for the full runs reported in the main text.

\paragraph{Brittleness protocol details.}\label{app:brittleness-details}
\textit{Queries.} We evaluate 100 correctly classified queries for CIFAR-10/Fashion-MNIST and 50 for RTE.
\textit{Candidate pool.} For each query we restrict training candidates to the same class as the query \citep{singla2023simple}.
\textit{Subset sizes.} We vary $M$ over a grid (CIFAR-10: $M \in \{200, 400, 600, 800, 1000, 1200\}$; Fashion-MNIST: $M \in \{50, 100, 150, 200, 250, 300\}$; RTE: $M \in \{20, 40, 60, 80, 100, 120\}$.
\textit{Retraining.} After removing the top-$M$ points, we retrain from the same initialization with identical data ordering and per-epoch shuffling (to isolate the effect of removal).

\paragraph{Backdoor protocol details.}\label{app:backdoor-details}
Let $\mathcal{Q}$ denote the set of $100$ backdoored test queries. For each query $q \in \mathcal{Q}$, let $\mathcal{G}(q)$ be its ground-truth attribution set, \ie the $50$ poisoned training examples with the same base class as $q$.
For each query $q \in \mathcal{Q}$, let $\pi_q = (i_1, i_2, \ldots)$ denote the ranked list of training examples produced by a given method, and let $\operatorname{rank}_q(i)$ denote the position of training example $i$ in $\pi_q$.
For TDA baselines this ranking is obtained by sorting training examples by descending attribution score, whereas for our greedy information-theoretic methods it is given by the order of greedy selection.
Let $\mathcal{R}_K(q) = \{i : \operatorname{rank}_q(i) \le K\}$ denote the top-$K$ retrieved set.
We report,
\begin{equation}
    \mathrm{Recall@}50
    \;=\;
    \frac{1}{|\mathcal{Q}|}
    \sum_{q \in \mathcal{Q}}
    \frac{|\mathcal{R}_{50}(q) \cap \mathcal{G}(q)|}{|\mathcal{G}(q)|},
\end{equation}
which measures the fraction of ground-truth poisoned examples recovered in the top $50$, averaged over queries.
Let $r_q^\star = \min_{i \in \mathcal{G}(q)} \operatorname{rank}_q(i)$ denote the rank of the highest-ranked ground-truth poisoned example for query $q$.
We also report the cutoff-based mean reciprocal rank,
\begin{equation}
    \mathrm{MRR@}100
    \;=\;
    \frac{1}{|\mathcal{Q}|}
    \sum_{q \in \mathcal{Q}}
    \begin{cases}
        \displaystyle\frac{1}{r_q^\star},
        & \text{if } r_q^\star \le 100, \\[0.75em]
        0, & \text{otherwise},
    \end{cases}
\end{equation}
which is the reciprocal rank of the first ground-truth poisoned example for each query, truncated to zero if no such example appears in the top $100$.

\paragraph{Coreset protocol details.}\label{app:coreset-details}
We extract $500$ examples from the test set as the query set for attribution and report test accuracy on the remaining test examples after retraining on the selected coreset.
We vary coreset sizes over $M \in \{100, 200, 500, 1000, 2000, 5000\}$.
For each seed, the coreset-trained model is retrained from the same initialization as the full-data model, using the same training procedure (optimizer, learning rate schedule, number of epochs, and data augmentation).
For this experiment we use the same scalar multiclass measurement function as \textsc{TRAK} for all Jacobian-based methods and do not apply NTK reparameterization.
We additionally report a class-balanced variant of this experiment in \cref{fig:coreset_cifar10_cls}, where per-class quotas are enforced during selection; as discussed in \cref{sec:coreset}, this substantially improves all TDA baselines, yet both \textsc{InfoGain} variants remain strongest overall.

For the TDA baselines, attribution is defined pointwise, so we aggregate over the query set by averaging the per-query scores for each training example and then taking the top-$M$ examples.
Our greedy methods jointly optimise over all $N_{\textrm{test}}$ query points at each step, so the single-query greedy rules do not directly apply. Below we derive a tractable multi-query selection criterion for \textsc{InfoGain}; the corresponding \textsc{InfoLoss} derivation is analogous and is therefore omitted.
Let $\MPhi_* \in \mathbb{R}^{N_{\textrm{test}} \times K}$ stack the query features $\{\vphi_{*q}\}_{q=1}^{N_{\textrm{test}}}$ row-wise.
The exact AOI greedy rule for \textsc{InfoGain} is then,
\begin{equation}
    \vx^{(m)}
    \;\leftarrow\;
    \argmin_{\vx \in \data \setminus \activeset}
    \log \det \rnd{
        \MPhi_*
        \MS_{\activeset \cup \{\vx\}}^{-1}
        \MPhi_*^\T
    }.
\end{equation}

The objective depends on $\vx$ through the AOI precision $\MS_{\activeset \cup \{\vx\}}$. We isolate the $\vx$-dependent terms via Sherman-Morrison and the matrix determinant lemma, then approximate the expensive query-query interaction:
\begin{align}
   \log \det \rnd{
       \MPhi_*
       \MS_{\activeset \cup \{\vx\}}^{-1}
       \MPhi_*^\T
   }
   \shortintertext{\color{gray}\  \  $\triangleright$ expand $\MS_{\activeset \cup \{\vx\}}^{-1}$ via Sherman-Morrison}
   &= \log \det \rnd{
       \MPhi_* \MS_{\activeset}^{-1} \MPhi_*^\T
       -
       \frac{
           \rnd{\MPhi_* \MS_{\activeset}^{-1} \vphi_{\vx}}
           \rnd{\MPhi_* \MS_{\activeset}^{-1} \vphi_{\vx}}^\T
       }{
           \sigma^2 + v_{\vx}^{\activeset}
       }
   }
   \shortintertext{\color{gray}\  \  $\triangleright$ matrix determinant lemma}
   &= \underbrace{\log \det \rnd{\MPhi_* \MS_{\activeset}^{-1} \MPhi_*^\T}}_{\text{const.\ in } \vx}
   +
   \log \rnd{
       1 -
       \frac{
           \rnd{\MPhi_* \MS_{\activeset}^{-1} \vphi_{\vx}}^\T
           \rnd{\MPhi_* \MS_{\activeset}^{-1} \MPhi_*^\T}^{-1}
           \rnd{\MPhi_* \MS_{\activeset}^{-1} \vphi_{\vx}}
       }{
           \sigma^2 + v_{\vx}^{\activeset}
       }
   }
   \shortintertext{\color{gray}\  \  $\triangleright$ first term is constant in $\vx$; approximate $\rnd{\MPhi_* \MS_{\activeset}^{-1} \MPhi_*^\T}^{-1} \approx \MI$ (drop query-query interactions)}
   &\approx \mathrm{const.}
   +
   \log \rnd{
       1 -
       \frac{
           \lVert \MPhi_* \MS_{\activeset}^{-1} \vphi_{\vx} \rVert_2^2
       }{
           \sigma^2 + v_{\vx}^{\activeset}
       }
   }.
\end{align}
This yields our tractable multi-query \textsc{InfoGain} criterion,
\begin{equation}
    \vx^{(m)}
    \;\leftarrow\;
    \argmax_{\vx \in \data \setminus \activeset}
    \frac{
        \lVert \MPhi_* \MS_{\activeset}^{-1} \vphi_{\vx} \rVert_2^2
    }{
        \sigma^2 + v_{\vx}^{\activeset}
    }.
\end{equation}
Applying the same linear-response variance correction as in \cref{app:linear_repsonse} removes the candidate-dependent denominator and yields the approximate variant,
\begin{equation}
    \vx^{(m)}
    \;\leftarrow\;
    \argmax_{\vx \in \data \setminus \activeset}
    \lVert \MPhi_* \MS_{\activeset}^{-1} \vphi_{\vx} \rVert_2^2.
\end{equation}

\begin{figure}[H]
    \centering
    \includegraphics[width=0.4\linewidth]{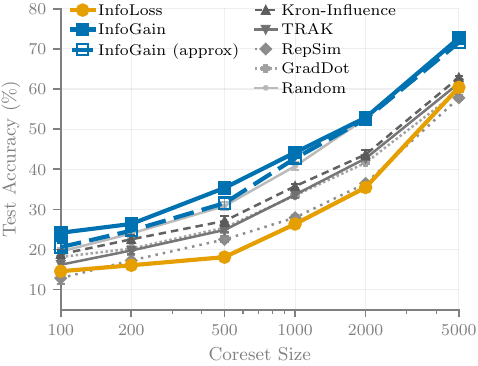}
    \caption{CIFAR-10 coreset selection with class-balanced TDA baselines. Enforcing equal per-class quotas substantially improves \textsc{TRAK}, \textsc{KronInfluence}, \textsc{GradDot} and \textsc{RepSim}, often lifting the strongest baselines above \textsc{InfoLoss}; however, both \textsc{InfoGain} variants remain strongest overall.}
    \label{fig:coreset_cifar10_cls}
\end{figure}

\clearpage
\begin{figure*}[!t]
    \section{VISUALIZATION OF TRAINING DATA ATTRIBUTION METHODS}\label{app:tda_visualization}
    \centering
    \begin{tabular}{@{}l@{}}
        \textsc{InfoLoss}\\
        \midrule
        \includegraphics{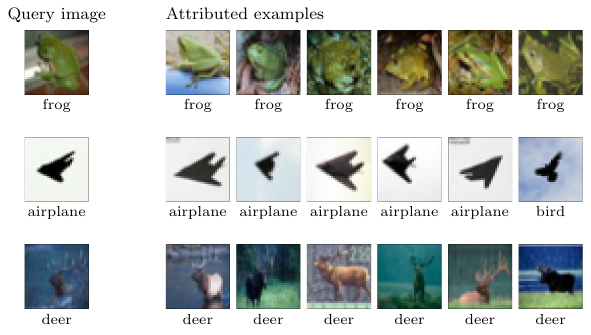}
    \end{tabular}
    
    \vspace{0.3cm}
    
    \begin{tabular}{@{}l@{}}
        \textsc{InfoGain}\\
        \midrule
        \includegraphics{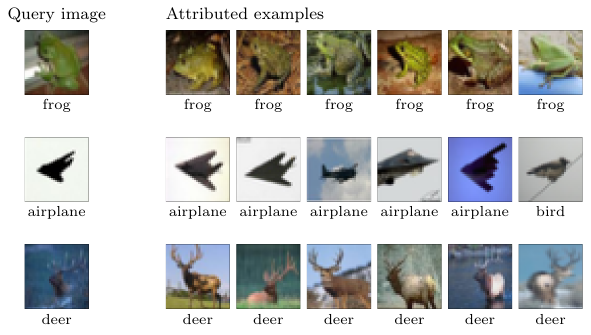}
    \end{tabular}
    
    \vspace{0.3cm}
    
    \begin{tabular}{@{}l@{}}
        \textsc{InfoGain (approx)}\\
        \midrule
        \includegraphics{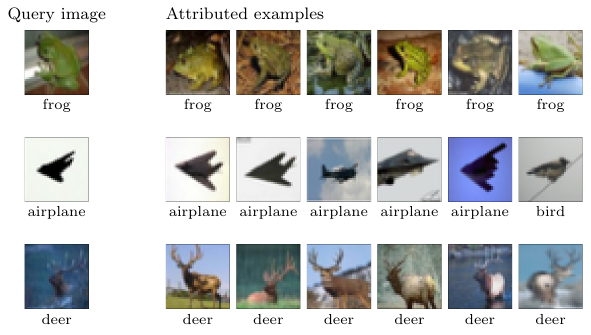}
    \end{tabular}
    \caption{Visualization of training data attribution using our Bayesian information-theoretic methods on the CIFAR-10 dataset. Top: \textsc{InfoLoss}, Middle: \textsc{InfoGain}, Bottom: \textsc{InfoGain (approx)}.}
    \label{fig:tda_methods_vis}
\end{figure*}

\begin{figure*}[!t]
    \centering
    \begin{tabular}{@{}l@{}}
        \textsc{KronInfluence}\\
        \midrule
        \includegraphics{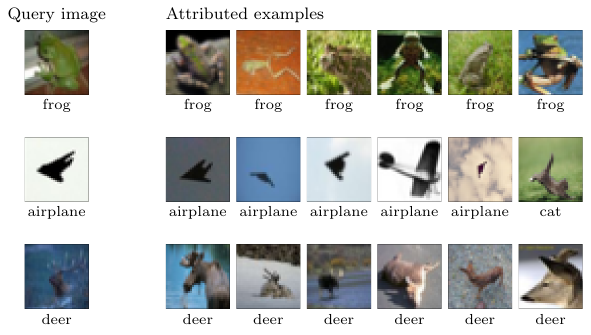}
    \end{tabular}
    
    \vspace{0.3cm}
    
    \begin{tabular}{@{}l@{}}
        \textsc{TRAK}\\
        \midrule
        \includegraphics{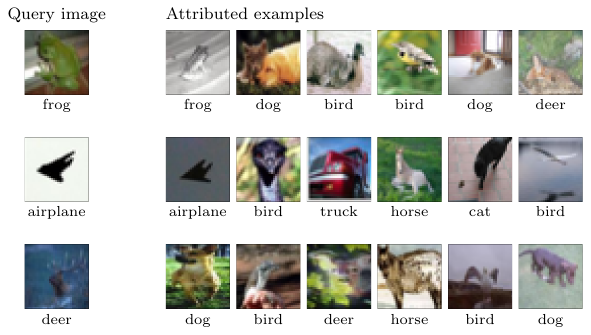}
    \end{tabular}
    
    \vspace{0.3cm}
    
    \begin{tabular}{@{}l@{}}
        \textsc{TracIn}\\
        \midrule
        \includegraphics{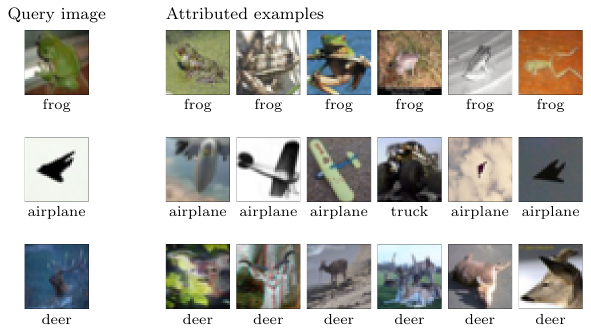}
    \end{tabular}
    \caption{Visualization of training data attribution using baseline influence-based methods on the CIFAR-10 dataset. Top: \textsc{KronInfluence}, Middle: \textsc{TRAK}, Bottom: \textsc{TracIn}.}
    \label{fig:tda_baselines_vis}
\end{figure*}

\begin{figure*}[!t]
    \centering
    \begin{tabular}{@{}l@{}}
        \textsc{RepSim}\\
        \midrule
        \includegraphics{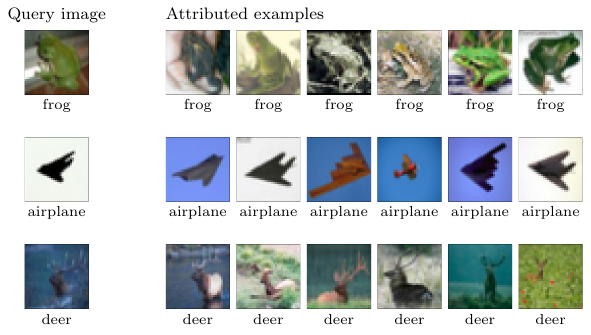}
    \end{tabular}
    \caption{Visualization of training data attribution using representational similarity (\textsc{RepSim}) on the CIFAR-10 dataset.}
    \label{fig:tda_repsim_vis}
\end{figure*}

\end{document}